\documentclass{article}
\usepackage{graphicx} % Required for inserting images

\usepackage{times}
\usepackage{soul}
\usepackage{url}
\usepackage[hidelinks]{hyperref}
\usepackage[utf8]{inputenc}
\usepackage[small]{caption}
\usepackage{graphicx}
\usepackage{amsmath}
\usepackage{amssymb}
\usepackage{mathtools}
\usepackage{amsthm}
\usepackage{booktabs}
\usepackage{algorithm}
\usepackage{algorithmic}
\usepackage{multirow}
\usepackage[switch]{lineno}

\usepackage{subfigure}
\usepackage{color}
\usepackage{authblk}

\usepackage{xcolor}    
\usepackage{algorithm}     
\usepackage{algorithmic}    
\usepackage{amsmath}      
\usepackage{booktabs}
\usepackage{threeparttable}
\usepackage[table]{xcolor}

\newtheorem{proposition}{Proposition}

\begin{document}

\title{Equivariant Efficient Joint Discrete and Continuous MeanFlow for Molecular Graph Generation}

\author[1]{Rongjian Xu}
\author[1]{Teng Pang}
\author[1]{Zhiqiang Dong}
\author[1]{Guoqiang Wu \textsuperscript{*}} 

\affil[*]{{\small Corresponding Author}}

\affil[1]{School of Software, Shandong University, Jinan, Shandong, China \\
rongjianxu50@gmail.com, silencept7@gmail.com, zhiqiangdong12@gmail.com}
\affil[*1]{School of Software, Shandong University, Jinan, Shandong, China; Luke EI, Jinan, China \\guoqiangwu90@gmail.com}

% \thanks{Corresponding Author}

\maketitle

\begin{abstract}
Graph-structured data jointly contain discrete topology and continuous geometry, which poses fundamental challenges for generative modeling due to heterogeneous distributions, incompatible noise dynamics, and the need for equivariant inductive biases. Existing flow-matching approaches for graph generation typically decouple structure from geometry, lack synchronized cross-domain dynamics, and rely on iterative sampling, often resulting in physically inconsistent molecular conformations and slow sampling.
To address these limitations, we propose Equivariant MeanFlow (EQUIMF), a unified SE(3)-equivariant generative framework that jointly models discrete and continuous components through synchronized MeanFlow dynamics. EQUIMF introduces a unified time bridge and average-velocity updates with mutual conditioning between structure and geometry, enabling efficient few-step generation while preserving physical consistency.
Moreover, we develop a novel discrete MeanFlow formulation with a simple yet effective parameterization to support efficient generation over discrete graph structures. Extensive experiments demonstrate that EQUIMF consistently outperforms prior diffusion and flow-matching methods in generation quality, physical validity, and sampling efficiency.
\end{abstract}

\section{Introduction}
Graph generation has become a central problem in modern machine learning, with broad applications in chemistry, biology, material science, and network analysis~\cite{deepgraphsurvey2022,schutt2017quantum,autoregressivediffusiongraph2024,ma2024exploring,molecularrnn2019}. 
% \guoqiang{more ref?}
Recent years have witnessed rapid progress in generative models for graphs, enabling the synthesis of complex relational structures with increasing fidelity and diversity.
% These advances have been driven by the development of powerful deep generative paradigms, including autoregressive models, variational methods, diffusion models, and more recently, flow-based approaches.

\noindent Graph-structured data is inherently discrete at its core: nodes and edges correspond to categorical entities and relational types, respectively. Accordingly, a prominent body of research centers on discrete graph generation, where the primary goal is to model probability distributions over valid node-edge configurations. Most existing methods leverage diffusion-based or flow-based frameworks—paradigms that enable principled likelihood-based training and improved scalability for large-scale graph dat. For example ~\cite{Vignac2022DiGress} proposed discrete-time discrete-state graph diffusion models. Subsequent works ~\cite {Xu2024DiSco,Siraudin2024Cometh} extended this line of research to continuous-time discrete-state graph diffusion. More recently, discrete flow-based graph generation models ~\cite{campbell2024generative,Qin2024DeFoG} have emerged, which are specifically designed to mitigate the computational inefficiencies inherent to graph diffusion models.  

\noindent For many real-world applications, including molecular generation, graph structure alone is insufficient: incorporating 3D geometric information has been shown to significantly improve generation quality, validity, and downstream performance, and equivariant architectures can further ensure geometric consistency under rigid transformations. However, existing flow and diffusion models~\cite{song2023equivariant,hoogeboom2022equivariant,liu2023flow,lipman2022flow} for discrete-continuous molecular generation often decouple topology from geometry, lacking cross-modal synchronized dynamics and inductive biases. This leads to physically inconsistent conformations and slow sampling, hindering their iterative generation process.

\noindent To address these drawbacks, we propose Equivariant MeanFlow (EQUIMF), a unified generative framework with $\mathbf{SE}(3)$ equivariance, which couples the modeling of discrete structural components and continuous geometric properties by leveraging synchronized MeanFlow dynamics. 
Specifically, we propose a new discrete MeanFlow model for efficient generation over discrete graph structures, leveraging a new, simple yet effective model parameterization.
Further, through a unified temporal alignment mechanism and mutual conditioning between structural and geometric representations, our approach enables the efficient generation of molecular structures in just a small number of steps.
Our results on a series of benchmarks show that EQUIMF consistently improves generation quality, physical validity, and sampling efficiency over almost 2x faster than state-of-the-art (SOTA) flow-matching and diffusion models. 
% \guoqiang{acceleration speed}

We summarize our main contributions as follows:
\begin{itemize}
    % \item Unified hybrid generation. We introduce a single framework that jointly models discrete graph structure and continuous 3D geometry under a synchronized time-bridge formulation.
    \item \textbf{New discrete MeanFlow}. We propose a new discrete MeanFlow model for discrete domains, which achieves efficient few-step sampling through a simple yet effective model parameterization strategy.
    \item \textbf{Hybrid MeanFlow with new mutual conditioning}. We propose a unified hybrid MeanFlow that jointly models discrete graph structures and continuous 3D geometries via a synchronized time-bridge and iterative mutual conditioning, supporting efficient sampling.
    \item \textbf{Equivariance-aware design and Empirical improvements}. We design an SE(3)-equivariant continuous head with theoretical symmetry guarantees, and our equivariant MeanFlow achieves superior performance on molecular generation benchmarks over SOTA flow/diffusion-based baselines.
    % \guoqiang{discrete meanflow}
\end{itemize}

\begin{figure*}[t]
    \centering
    \includegraphics[width=1.0\textwidth]{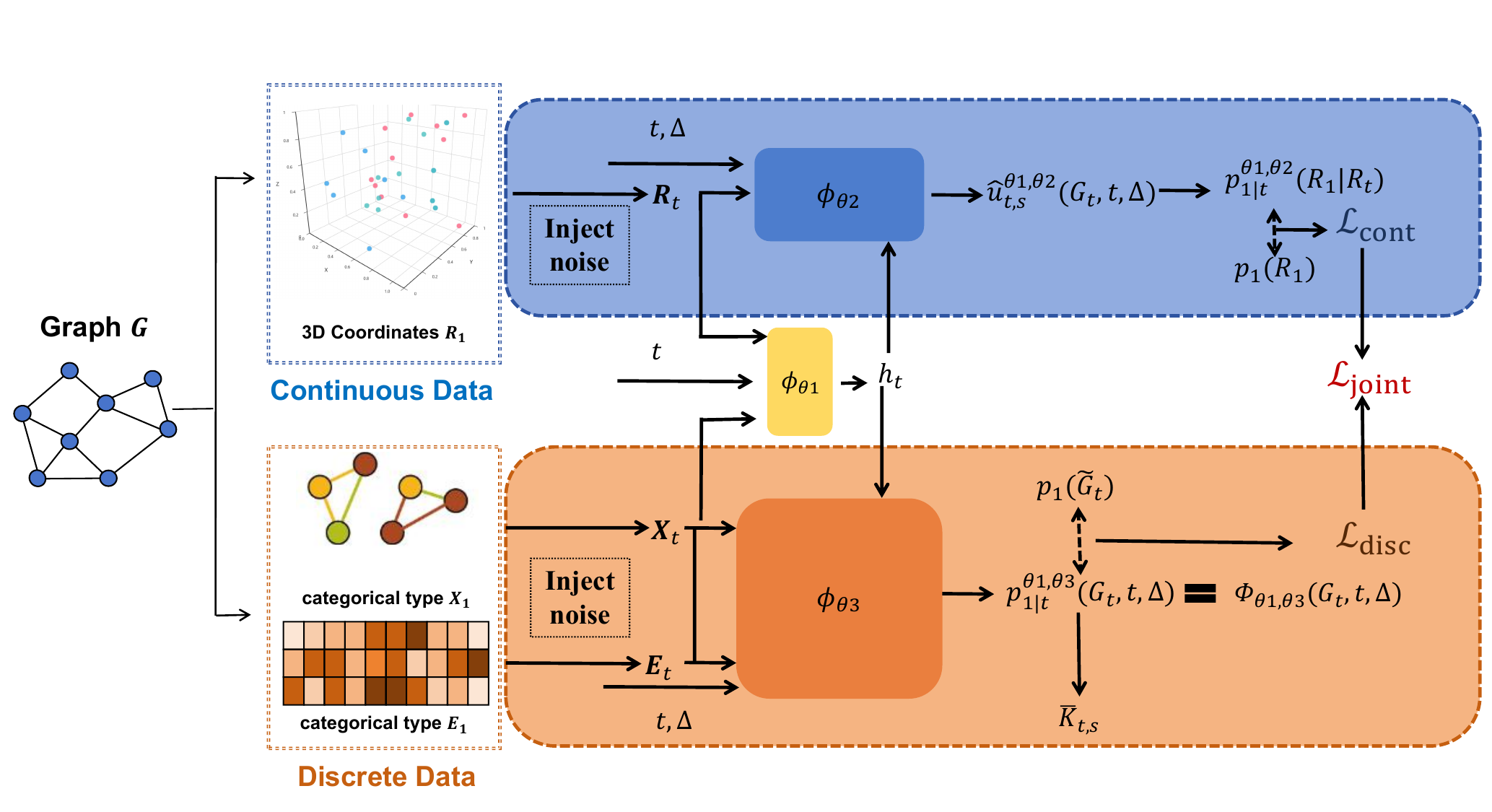}
    \caption{Architecture of Equivariant MeanFlow. This framework jointly models discrete graph structures ($\widetilde{\mathbf{G}} = (\mathbf{X}, \mathbf{E})$) and continuous 3D molecular geometries ($\mathbf{R}$) via a shared SE(3)-equivariant encoder $\phi_{\theta_1}$. The encoder’s outputs drive a discrete MeanFlow head $\phi_{\theta_3}$ and a continuous one $\phi_{\theta_2}$, which are optimized via a joint loss for mutually conditioned generation.}
    \label{fig:Structure}
\end{figure*}

\section{Preliminaries and Background}

\subsection{Problem Setup}

% \paragraph{Notation.}
\noindent\textbf{Notations.}
In this paper, a plain lowercase letter (e.g., $a$) represents a scalar, a bold
lowercase letter (e.g., $\mathbf{a}$) denotes a vector, and a bold uppercase
letter (e.g., $\mathbf{A}$) denotes a matrix.  
For a vector $\mathbf{a}$, its $i$-th entry is written as $a_i$.  
% For a matrix $\mathbf{A}$, we use $\mathbf{A}_{i,:}$ and $\mathbf{A}_{:,j}$ to refer to its $i$-th row and $j$-th column, and $A_{ij}$ for the $(i,j)$-th element.  
% Blackboard-bold symbols (e.g., $\mathbb{A}$) denote sets, and $|\mathbb{A}|$ is the cardinality of $\mathbb{A}$.  
Let $[n] = \{1,2,\dots,n\}$ be the index set of the first $n$ integers.
% The indicator of an event or condition $\mathcal{E}$ is denoted by $\mathbf{1}[\mathcal{E}]$, which equals $1$ if $\mathcal{E}$ holds and $0$ otherwise.

% \paragraph{Setting.}
\noindent\textbf{Problem Setup.}
We consider attributed (geometric) graphs with categorical node/edge types and 3D coordinates of all nodes.  
A graph with $n$ nodes is denoted by $\mathbf{G} = (\mathbf{X}, \mathbf{E}, \mathbf{R})$, 
% \[
%     \mathbf{G} = (\mathbf{V}, \mathbf{E}, \mathbf{X}),
% \]
% \begin{align*}
%     \mathbf{G} = (\mathbf{V}, \mathbf{E}, \mathbf{X}),
% \end{align*}
where $\mathbf{X} = \bigl(x^{(i)}\bigr)_{i \in [n]} \in [b]^{n}, \mathbf{E} = \bigl(e^{(ij)}\bigr)_{i,j \in [n]} \in [a+1]^{n \times n}$, and $\mathbf{R} = \bigl(\mathbf{r}^{(i)}\bigr)_{i \in [n]} \in \mathbb{R}^{n \times 3}$. \footnote{Note that here we denote nodes as $\mathbf{X}$ instead of $\mathbf{x}$ to be more compatible with $\mathbf{E}$ and $\mathbf{R}$, thereby helping easy readability.}
Here $a$ and $b$ denote the numbers of \emph{actual} edge and node types, respectively.
We treat the absence of an edge as a special edge type, so the total number of edge types is $a+1$.
The matrix $\mathbf{R}$ collects 3D coordinates of all nodes, where each row $\mathbf{r}^{(i)} \in \mathbb{R}^3$ represents the spatial position of node $i$. For simplicity, denote $\widetilde{\mathbf{G}} = (\mathbf{X}, \mathbf{E})$ for a (geometric) graph $\mathbf{G}$.
Given a training set $\mathcal{D} = \{\mathbf{G}^{(i)}\}_{i=1}^{N}$ with $N$ graphs, where each graph $\mathbf{G}^{(i)} \overset{\text{i.i.d.}}{\sim} p_{\text{data}}$, the goal is to learn a generative model (or generator) 
to match the (unknown) data distribution $p_{\text{data}}$ over graphs.

\subsection{Flow Matching For Graph Generation}
Flow Matching (FM)~\cite{lipman2022flow,albergo2023stochastic,liu2023flow} aims to learn how to transform a simple, easy-to-sample noise distribution \(p_0 = p_\epsilon\) into a target data distribution \(p_1 = p_{\text{data}}\). It is originally proposed for the generative modeling of continuous domains (e.g., $\mathbb{R}$) via an Ordinary Differential Equation (ODE) $\frac{\mathrm{d} z_t}{\mathrm{d} t} = u (z_t, t), \ t \in [0,1]$.

For discrete random variables $z \in [b]$ (e.g., graph node types), Discrete Flow Matching (DFM)~\cite{campbell2024generative,Qin2024DeFoG} provides an elegant framework by modeling the generation process as a Continuous-Time Markov Chain (CTMC), leading to the Kolmogorov equation, an ODE (a.k.a., probability flow): $\partial_t \mathbf{p}_t = \mathbf{K}_t^{\top} \mathbf{p}_t$, where $t \in [0, 1]$, $\mathbf{p}_t = [p_t(z_1=1,...,p(z_1=b)]$, and $\mathbf{K}_t(\cdot, \cdot)$ is the (instantaneous) transition rate matrix.
In the noisy process, DFM defines a deterministic noising trajectory that starts from a data point $z_1 \sim p_1$ and gradually interpolates to an initial noise distribution $p_0$ (typically a uniform distribution over the discrete domain), given by
% This trajectory is given by:
$$
p_{t|1}(z_t|z_1) = t\,\delta(z_t,z_1) + (1-t)\,p_0(z_t),
$$
where $\delta(z_t,z_1)$ is the Kronecker delta function ($\delta = 1$ if $z_t = z_1$ and $0$ otherwise).  When $t=1$, the distribution transforms to the original data point $z_1$; as $t$ decreases to $0$, it smoothly evolves to the initial noise $p_0$.

In the denoising process, to generate new samples, DFM considers a conditional transition rate matrix $\mathbf{K}_{t|1}(\cdot,\cdot|z_1)$ to reverse the noising trajectory, which can be defined as: 
\begin{align*}
    \mathbf{K}_{t|1}(z_t, z_{t+\mathrm{d}t} | z_1) = \frac{\text{ReLU}\left[ \partial_t p_{t|1}(z_{t+\mathrm{d}t} | z_1) - \partial_t p_{t|1}(z_t | z_1) \right]}{b \cdot p_{t|1}(z_t | z_1)}
\end{align*}
where $\text{ReLU} (x) = \max (0, x)$ and $\mathrm{d}t$ is next infinitesimal time.
Then, it has $\mathbf{K}_t(z_t, \cdot) = \mathbb{E}_{p_{1|t}(z_1 | z_t)} \left[ \mathbf{K}_{t|1}(z_t, \cdot | z_1) \right]$.
% $\mathbf{K}_{t|1}(\cdot,\cdot|z_1)$
% \guoqiang{condition?}that reverses the noising trajectory.
This rate matrix $\mathbf{K}_t$ governs the CTMC dynamics, which allows us to start from a sample drawn from $p_0$ and evolve it to recover the data distribution $p_{\text{data}}$, where the transition probability between discrete states is given by: 
\begin{align}
\label{eq:transition}
    p_{t + \mathrm{d}t|t}(z_{t+\mathrm{d}t}|z_t) = \delta(z_t,z_{t+\mathrm{d}t}) + \mathbf{K}_t(z_t,z_{t+\mathrm{d}t}) \mathrm{d}t,
\end{align}
\noindent where $\mathrm{d}t$ is replaced with a finite time interval $\Delta$ in practice.

\noindent Finally, to (implicitly) model the instantaneous rate matrix $\mathbf{K}_t$, DFM learns $p_{1|t}(z_1 | z_t)$ which is parameterized by a neural network $\phi (z_t, t)$. 

\noindent While DFM shows promise in molecular graph generation, it has two critical limitations.
First, it (implicitly) models the instantaneous rate matrix, which requires numerous fine-grained time steps, leading to low sampling efficiency. 
Second, most DFM-based models only focus on discrete graph structures, ignoring continuous 3D geometric information and producing geometrically inconsistent molecules. 
These challenges motivate our unified generative framework, which accelerates sampling while jointly modeling discrete structure and continuous geometry.

\subsection{Equivariance}

Let \(G\) be a group acting on an input space \(\mathcal{X}\) and an output space
\(\mathcal{Y}\).  
For each \(g \in G\), let \(\mathbf{S}_g : \mathcal{X} \to \mathcal{X}\) and
\(\mathbf{T}_g : \mathcal{Y} \to \mathcal{Y}\) denote the corresponding linear
representation operators.  
A function \(f : \mathcal{X} \to \mathcal{Y}\) is said to be \emph{equivariant} to the
action of \(G\) if
$\mathbf{T}_g\bigl(f(\mathbf{x})\bigr) = f\bigl(\mathbf{S}_g(\mathbf{x})\bigr), \forall g \in G,\ \mathbf{x} \in \mathcal{X}$.
If \(\mathbf{T}_g\) is the identity mapping for all $g$, the function $f$ is \emph{invariant} to the group action.

\noindent In molecular and geometric modeling, the relevant symmetry group is the
Euclidean group $\mathbf{SE}(3)$, generated by translations, rotations, and reflections in $\mathbb{R}^3$. 
An element $g = (\mathbf{Q}, \mathbf{a}) \in \mathbf{SE}(3)$ acts on a point $\mathbf{r} \in \mathbb{R}^3$ as \( \mathbf{S}_g(\mathbf{r})\) = $\mathbf{Q} \mathbf{r} + \mathbf{a}$, 
where $\mathbf{Q} \in \mathbb{R}^{3 \times 3}$ is an orthogonal matrix ($\mathbf{Q}^\top \mathbf{Q} = \mathbf{I}$) and $\mathbf{a} \in \mathbb{R}^3$ is a translation vector. 
% \guoqiang{problem with the notation $R$ and $a$? $Q$ denotes orthogonal matrix?}
For the function $f$ that outputs geometric quantities in $\mathbb{R}^3$, $\mathbf{SE}(3)$-equivariance requires 
$f(\mathbf{Q} \mathbf{r} + \mathbf{a}) = \mathbf{Q} f(\mathbf{r}) + \mathbf{a}$,which is the modeling core.
To model equivariant distributions over molecular graphs within our MeanFlow framework, 
we adopt widely-used \emph{E(3)-Equivariant Graph Neural Networks} (EGNNs)~\cite{garcia2021en}, a type of graph neural network that satisfies the equivariance constraint as our backbone.

\section{Method}
% \section{Equivariant Continuous and Discrete Meanflow}
In this section, we introduce the Equivariant MeanFlow (EQUIMF), a unified $\textbf{SE}(3)$-equivariant generative framework that jointly models discrete and continuous domains via synchronized MeanFlow dynamics. 

\subsection{Noising and Denoising Process with Unified Time}\label{sec:4.1}
We model the dynamics of graph $\mathbf{G}_t=(\mathbf{X}_t,\mathbf{E}_t,\mathbf{R}_t)$ with a synchronized time trajectory, where discrete structure $\widetilde{\mathbf{G}}_t = (\mathbf{X}_t, \mathbf{E}_t)$, and continuous geometry $\mathbf{R}_t$ are constructed by discrete and continuous MeanFlow respectively, which we will discuss in details in subsequent sections.

% \noindent\textbf{Noising Process}.
% For molecular graphs, a synchronized noising process jointly perturbs the discrete structure and continuous geometry with $p_{t|1}(\mathbf{G}_t | \mathbf{G}_1 ) = p_{t|1}(\widetilde{\mathbf{G}}_t | \widetilde{\mathbf{G}}_1 ) p_{t|1}(\mathbf{R}_t | \mathbf{R}_1)$, as follows.

\noindent\textbf{Noising Process}.
For molecular graphs, a synchronized noising process jointly perturbs the discrete structure and continuous geometry, with
\begin{align*}
p_{t|1}(\mathbf{G}_t \mid \mathbf{G}_1)
= p_{t|1}(\widetilde{\mathbf{G}}_t \mid \widetilde{\mathbf{G}}_1) \, p_{t|1}(\mathbf{R}_t \mid \mathbf{R}_1),
\end{align*}
as follows.
% We construct a synchronized forward noising process for molecular graphs with $p_{t|1}(\mathbf{G}_t | \mathbf{G}_1 ) = p_{t|1}(\widetilde{\mathbf{G}}_t | \widetilde{\mathbf{G}}_1 ) p_{t|1}(\mathbf{R}_t | \mathbf{R}_1)$, which jointly perturbs the discrete structure and continuous geometry as follows.
% \begin{itemize}
    % \item \textbf{Sample Graph Data}: Sample a complete molecular graph from our dataset, consisting of discrete node features $\mathbf{X}_1$, edge features $\mathbf{E}_1$, and continuous atomic coordinates $\mathbf{R}_1$:
    % $$
    % (\mathbf{X}_1, \mathbf{E}_1, \mathbf{R}_1) \sim p_1.
    % $$
    
    \noindent 1.~\textbf{Sample Time Pair}. Uniformly sample a base time step $t \sim \mathcal{U}(0, 1)$ and a small time interval $\Delta$, such that the subsequent time step $s = t + \Delta$ satisfies $s \le 1$.
    To guarantee this, we enforce $\Delta \in \left[\Delta_{\min}, 1-t\right]$, where $\Delta_{\min}$ is a hyperparameter controlling the minimum interval length.

    \noindent 2.~\textbf{Inject Noise into Discrete Structure}.
    For discrete structure $\widetilde{\mathbf{G}}_t$, we inject noise via DFM trajectory, which interpolates between the data point and a noise distribution $p_0$:
    \vspace{-2pt}
    \begin{align*}
        p_{t|1}(\widetilde{\mathbf{G}}_t  | \widetilde{\mathbf{G}}_1 )  =   
        \prod_{i \in [N]} p_{t|1}\left(x_t^{(i)} | x_1^{(i)}\right) \prod_{i<j} p_{t|1}\left(e_t^{(ij)} | e_1^{(ij)}\right),
    \end{align*}
    where for each node $i \in [N]$, $p_{t|1}\left(x_t^{(i)} | x_1^{(i)}\right) = t \delta\left(x_t^{(i)}, x_1^{(i)}\right) + (1-t) p_0\left(x_t^{(i)}\right)$, and $p_{t|1}\left(e_t^{(ij)} | e_1^{(ij)}\right)$ is similarly defined for each edge $e_t^{(ij)}$. 
    
    \noindent 3.~\textbf{Inject Noise into Continuous Coordinates}.
    For the continuous atomic coordinates $\mathbf{R}_t$ (or $\mathbf{R}_s$), we inject noise using a linear interpolation schedule:
    \begin{align*}
    p_{t|1}(\mathbf{R}_t | \mathbf{R}_1) &= \prod_{i \in [N]} p_{t|1} \left(\mathbf{r}_t^{(i)} | \mathbf{r}_1^{(i)}\right) ,
    \end{align*}
    where for each $i \in [N]$, $\mathbf{r}_t^{(i)} = t \mathbf{r}_1^{(i)} + (1 - t) \mathbf{\epsilon}, \ \mathbf{\epsilon} \sim \mathcal{N}(\mathbf{0}, I)$. Similarly, we define the counterpart $p_{s|1}(\mathbf{R}_s | \mathbf{R}_1)$ of $\mathbf{R}_s$.

\noindent\textbf{Denoising Process with New Conditional Generation}.
% \begin{align}
%     p_{1|t}^{\theta} (\mathbf{G}_1 | \mathbf{G}_t) = p_{1|t}^{\theta_1, \theta_3} (\widetilde{\mathbf{G}}_1 | \mathbf{G}_t) p_{1|t}^{\theta_1, \theta_2} (\mathbf{R}_1 | \mathbf{G}_t) 
% \end{align}
In the denoising (or sampling) process, we model the distribution $p_{1|t} (\mathbf{G}_1 | \mathbf{G}_t)$ according to the following formula:
\begin{align}
\label{eq:condition_key}
    \hat{p}_{1|t} (\mathbf{G}_1 | \mathbf{G}_t) := \hat{p}_{1|t} (\widetilde{\mathbf{G}}_1 | \mathbf{G}_t) \hat{p}_{1|t} (\mathbf{R}_1 | \mathbf{G}_t) .
\end{align}
Note that this is different from the following decomposition in~\cite{campbell2024generative}:
\begin{align}
    \hat{p}_{1|t} (\mathbf{G}_1 | \mathbf{G}_t) := \hat{p}_{1|t} (\widetilde{\mathbf{G}}_1 | \widetilde{\mathbf{G}}_t) \hat{p}_{1|t} (\mathbf{R}_1 | \mathbf{R}_t) ,
\end{align}
which denoises the discrete part $\widetilde{\mathbf{G}}$ and continuous one $\mathbf{R}$ independently, except with shared time $t$. 
In contrast, our modeling approach allows each part to evolve not only conditioned on its own information of the preceding time but also on that of the other part. 
Intuitively, this can enable us to better keep the overall harmony of (geometric) graphs, including discrete structure and continuous geometry, which is also verified by our experimental results (see Sec.~\ref{sec:experiments}).
% For the convenience of discussions, we explicitly express the parameters related to the above parameterized distributions.

\noindent We parameterized the above distributions by neural networks, where $\theta_1$ denotes shared parameters for modeling both discrete and continuous parts, and $\theta_2$ denotes unique parameters for the continuous part, while $\theta_2$ denotes those of the discrete one, and $\theta = (\theta_1, \theta_2, \theta_3)$. We will detail it in the subsequent sections (see Fig.~\ref{fig:Structure} for the overall architecture).

% \guoqiang{clear parameter} 

% logic: representation (or modelling) - noising - denoising - train - sampling - architecture - mutual condition

\subsection{Equivariant MeanFlow for Continuous Geometry}\label{sec:4.3}
% \guoqiang{need to revise? This section is not clear.}
We introduce a conditional generative model for continuous graph structures based on the continuous MeanFlow~\cite{geng2025mean}, which defines the model parameterization to encode an average velocity field, in contrast to the instantaneous velocity fields typically employed in flow-based generative models. This allows us to frame the generative process as a trajectory governed by an ODE
$\dot{{\mathbf{{R}}}}_t = {u}_t((\widetilde{\mathbf{G}}_t, \mathbf{R}_t), t) = {u}_t(\mathbf{G}_t, t)$.~\footnote{Notably, we vectorize the matrix $\mathbf{R}_t$ and here we can view it as a vector.}
% $\dot{{\mathbf{{R}}}}_t ={u}_t(\mathbf{R_t}|\mathbf{R_1},\mathbf{R_0})$.
% In this section, we detail the MeanFlow framework for modeling \textbf{SE}(3)-equivariant continuous atomic coordinates in molecular graph generation. MeanFlow learns an averaged velocity field based on the instantaneous velocity field. It is originally proposed for the generative modeling of continuous domains via an Ordinary Differential Equation (ODE) $\frac{\mathrm{d} \mathbf{R}_t}{\mathrm{d} t} = u (\mathbf{R}_t, t)$.

% We obtain pairs of intermediate coordinate states \( \mathbf{R}_t, \mathbf{R}_s \in \mathbb{R}^{N \times 3} \), where \( s = t + \Delta \), for the same molecule at times \(t\) and \(s\), respectively.
\noindent MeanFlow learns an averaged velocity field based on the instantaneous velocity field, 
where the target average velocity ${u}_{t,s}(\mathbf{G}_t, t, s)$  between the interval $[t,s]$ is defined as:
\begin{align*}
    {u}_{t,s}(\mathbf{G}_t, t, s) = \frac{1}{s-t}\int_t^{s} {u}_\beta(\mathbf{G}_\beta, \beta) \mathrm{d} \beta.    
\end{align*}
Then, we train a parameterized neural network $\hat{{u}}_{t,s}^{\theta_1,\theta_2}(\mathbf{G}_t, t, \Delta)$ to approximate the target average velocity field with $s = t + \Delta$, where the training loss is
\begin{align}
\label{eq:cont_loss}
    \mathcal{L}_{\text{cont}} & (\theta_1, \theta_2) =  \mathbb{E}_{\mathbf{R}_1,\mathbf{R}_0,t,s} \bigg[ w(t,s) \times \notag \\
    & \left\| \hat{{u}}^{\theta_1, \theta_2}_{t,s}(\mathbf{G}_t, t, \Delta) - \text{sg}\left({u}_{t,s}^{\text{tgt}}(\mathbf{G}_t, t, \Delta) \right) \right\|^2 \bigg], 
\end{align}
where $w(t,s)$ is a sampling weight, $\text{sg}$ denotes a stop-gradient operation, and the target average velocity ${u}_{t,s}^{\text{tgt}}$ is
% the ${u}_{t,s}^{\text{tgt}}(\mathbf{G}_t,\Delta)$ is defined as:
\begin{align*}
&{u}_{t,s}^{\text{tgt}}(\mathbf{G}_t, t, \Delta) := u_t(\mathbf{G}_t, t) + \Delta \bigg[ \\
& \nabla_{\mathbf{R}_t} \hat{{u}}^{\theta_1, \theta_2}_{t,s}(\mathbf{G}_t, t, \Delta) \cdot u_t(\mathbf{G}_t, t) + \partial_t \hat{{u}}^{\theta_1, \theta_2}_{t,s}(\mathbf{G}_t, t, \Delta) \bigg],
\end{align*}
% with \text{sg} denoting a stop-gradient operation. The $w(t,s)$ is a sampling weight over time pairs $(t,s)$. 
which can be efficiently computed by the Jacobian-vector products (JVP). Based on the approximate average velocity, we can transport $\mathbf{R}_t$ to $\mathbf{R}_s$ by 
\begin{align}
\label{eq:for_sample}
    \mathbf{R}_s = \mathbf{R}_t + \Delta \times \hat{{u}}^{\theta_1, \theta_2}_{t,s}(\mathbf{G}_t, t, \Delta).
\end{align}
For convenience, we denote the induced corresponding distribution by $\hat{{u}}^{\theta_1, \theta_2}_{t,s}(\mathbf{G}_t, t, \Delta)$ as $p_{1|t}^{\theta_1, \theta_2} (\mathbf{R}_1 | \mathbf{G}_t)$.

\noindent\textbf{Equivariance}. Our parameterized neural network is based on the EGNN~\cite{garcia2021en},  a type of graph neural network that satisfies the equivariance constraint as our backbone. Therefore, our model keeps the equivariance.

\subsection{Discrete MeanFlow for Molecular Graph Structure}
\label{sec:4.4}

Here, we propose a new discrete MeanFlow model to generate discrete graph structures, where the core idea is to provide a new model parameterization for modeling the average transition rate matrix rather than the instantaneous one in discrete flow matching (DFM). 
For simplicity, we discuss only one dimension $x^{(i)} \in [b]$ of discrete node structures $\mathbf{X}$ and omit the similar edge ones $\mathbf{E}$.  

\noindent Recall that in DFM, the temporal evolution of the marginal distribution over discrete (node) states is characterized by an ODE that governs the conservation law of probability mass: $\partial_t \mathbf{p}_t^{(i)} = {\mathbf{K}_t^{(i)}}^{\top} \mathbf{p}_t^{(i)}$, where $\mathbf{p}_t^{(i)} = [p_t(x^{(i)} = 1), \dots, p_t(x^{(i)} = b)]$, and $\mathbf{K}_t^{(i)} \in \mathbb{R}^{b \times b}$ is the instantaneous transition rate matrix.
% \begin{align*}
%     \partial_t \mathbf{p}_t = \mathbf{K}_t^{\top} \mathbf{p}_t
% \end{align*}
In the practical denoising (or sampling) process, the transition probability between discrete states is given by Eq.~\eqref{eq:transition} with the finite time interval $\Delta$. 

\noindent\textbf{New Parameterization}.
In DFM, it learns $p_{1|t}$ parameterized by a neural network $\phi (\mathbf{G}_t, t)$, to (implicitly) model the instantaneous rate matrix $\mathbf{K}_t^{(i)}$. 
However, this requires an extremely small time interval $\Delta$, leading to a low convergence rate and sampling efficiency. Intuitively, a finite time interval can introduce an average transition rate matrix as 
% which might accelerate the sampling process, with its definition as 
\begin{align*}
\bar{\mathbf{K}}_{t,s}^{(i)} (\cdot, \cdot) := \frac{1}{\Delta} \int_t^s \mathbf{K}_{\beta}^{(i)} (\cdot, \cdot)  d\beta.
\end{align*}
If this average matrix could be modeled well, it is expected to use a larger time interval than that of the instantaneous one, thereby accelerating the sampling process.
Therefore, based on this insight, to model the average rate matrix $\bar{\mathbf{K}}_{t,s}^{(i)}$, we learn $p_{1|t}$ parameterized by a neural network $\phi(\mathbf{G}_t, t, \Delta)$ with additional input $\Delta$ compared with the one in DFM. Similarly to DFM, the training loss with a weighting $\lambda$ is
% \begin{align}
% \label{eq:disc_loss}
%     \mathcal{L}_{\text{disc}} = &  \mathbb{E}_{t, p_1(\widetilde{\mathbf{G}}_1),  p_{t|1}(\widetilde{\mathbf{G}}_t|\widetilde{\mathbf{G}}_1)} 
%     % \text{CE}_\lambda\!\left(\widetilde{\mathbf{G}}_1, p_{1|t}^\theta(\cdot | \widetilde{\mathbf{G}}_t)\right),
%     \Bigg[ - \sum_{i} \log \left(p_{1|t}^{\theta,(i)}\left(x_1^{(i)} | \mathbf{G}_t\right)\right) \notag \\
%     & - \lambda \sum_{i<j} \log \left(p_{1|t}^{\theta,(ij)}\left(e_1^{(ij)} | \mathbf{G}_t\right)\right) \Bigg].
% \end{align}
\begin{align}
\label{eq:disc_loss}
    \mathcal{L}_{\text{disc}} (\theta_1, \theta_3) & =   \mathbb{E}_{t, p_1(\mathbf{G}_1),  p_{t|1}(\mathbf{G}_t | \mathbf{G}_1)} 
    % \text{CE}_\lambda\!\left(\widetilde{\mathbf{G}}_1, p_{1|t}^\theta(\cdot | \widetilde{\mathbf{G}}_t)\right),
    \Big[ - \log p_{1|t}^{\theta_1, \theta_3}\left(  \mathbf{X}_1 | \mathbf{G}_t\right) \notag \\
    & \qquad \quad - \lambda \log p_{1|t}^{\theta_1, \theta_3}\left( \mathbf{E}_1 | \mathbf{G}_t\right) \Big] ,
\end{align}
where $p_{1|t}^{\theta_1, \theta_3}\left( \mathbf{X}_1 | \mathbf{G}_t\right) = \prod_i p_{1|t}^{\theta_1, \theta_3}\left(x_1^{(i)} | \mathbf{G}_t\right)$ and $p_{1|t}^{\theta_1, \theta_3}\left( \mathbf{E}_1 | \mathbf{G}_t\right) = \prod_{i \neq j} p_{1|t}^{\theta_1, \theta_3}\left( e_1^{(ij)} | \mathbf{G}_t \right)$.
\newcommand{\algcomment}[1]{
  \hfill
  \textcolor{gray!60!black}{$\triangleright$ \footnotesize #1}%
}
\begin{algorithm}[H]
\caption{Training Algorithm of EQUIMF}
\label{alg:train}
\begin{algorithmic}[1]
\REQUIRE Graph dataset \( \mathcal{D} = \{\mathbf{G}_1, \dots, \mathbf{G}_N\} \).
\STATE \textbf{Initialization}: $\theta, \Delta_{\min}$ 
% \guoqiang{?}
\WHILE{$\theta$ not converged}
    \STATE Sampling $\mathbf{G}_1 \sim \mathcal{D}$, $t \sim \mathcal{U}(0,1)$, $\Delta \sim \mathcal{U}(\Delta_{\min},1-t)$
    \STATE Sampling $\widetilde{\mathbf{G}}_t \sim p_{t|1}(\widetilde{\mathbf{G}}_t  | \widetilde{\mathbf{G}}_1 )$  \algcomment{Noising}
    \STATE Sampling $\mathbf{R}_t \sim p_{t|1}(\mathbf{R}_t | \mathbf{R}_1)$ \algcomment{Noising}
    % \STATE $\mathbf{h}_t^{\text{disc}},\mathbf{h}_t^{\text{cont}} \leftarrow \phi_{\theta_1}(\mathbf{G}_t, t)$
    % \STATE \(\hat{\mathbf{u}}_{t\to s} \leftarrow f_{\text{cont}}(\mathbf{h}_t^{\text{cont}}, \mathbf{R}_t, \Delta) \in \mathbb{R}^{N\times 3}\)
    % \STATE \( \mathbf{\bar{K}}_{t \to s} \leftarrow f_{\text{disc}}(\mathbf{h}_t^{\text{disc}}, \Delta) \)
    \STATE $p_{1|t}^{\theta}(\cdot | \mathbf{R}_t) \leftarrow \phi_{\theta_1,\theta_2}({\mathbf{G}}_t, \Delta)$  \algcomment{Denoising Prediction}
    \STATE $p_{1|t}^{\theta}(\cdot | \widetilde{\mathbf{G}}_t) \leftarrow \phi_{\theta_1,\theta_3}(\mathbf{G}_t,\Delta)$ \algcomment{Denoising Prediction}
    \STATE $\mathcal{L}_{\text{cont}} \leftarrow$ Eq.~\eqref{eq:cont_loss}
    \STATE $\mathcal{L}_{\text{disc}} \leftarrow$ Eq.~\eqref{eq:disc_loss}
    \STATE $\mathcal{L}_{\text{joint}} \leftarrow$ Eq.~\eqref{eq:joint_loss}
    \STATE $\theta \leftarrow \text{Optimize}(\mathcal{L}_{\text{joint}})$
\ENDWHILE

\STATE \textbf{Return} $\theta$
\end{algorithmic}
\end{algorithm}

\begin{algorithm}[H]
\caption{Sampling Algorithm of EQUIMF}
\label{alg:sampling}
\begin{algorithmic}[1]
\REQUIRE Trained models $\phi_{\theta_1}, \phi_{\theta_2}, \phi_{\theta_3}$.
\ENSURE Generated graph \( \mathbf{G}_1 = (\mathbf{X}_1, \mathbf{E}_1, \mathbf{R}_1) \)
\STATE Sampling noised data: \( \mathbf{X}_0, \mathbf{E}_0 \sim p_0 \); \( \mathbf{R}_0 \sim \mathcal{N}(0, I) \);
\STATE \( t \leftarrow 0 \)
\FOR{$t = 0$ to $1 - \Delta$ with step $\Delta$}
    \STATE \( s \leftarrow t + \Delta\)
    % \STATE \( \mathbf{h}_t^{\text{disc}}, \mathbf{h}_t^{\text{cont}} \leftarrow \phi_\theta(\mathbf{G}_t, t) \)
    \STATE $\hat{{u}}_{t,s}^{\theta_1,\theta_2}(\mathbf{G}_t, t, \Delta) \leftarrow \phi_{\theta_1,\theta_2}({\mathbf{G}}_t, \Delta)$  \algcomment{Prediction}
    \STATE $p_{1|t}^{\theta}(\cdot | \widetilde{\mathbf{G}}_t) \leftarrow \phi_{\theta_1,\theta_3}(\{\mathbf{G}_t, \Delta )$   \algcomment{Prediction}
    \STATE \(\mathbf{R}_s \leftarrow \) Calculate with Eq.~\eqref{eq:for_sample} \algcomment{Denoising}
    \STATE $ \widetilde{\mathbf{G}}_s \leftarrow $ Sample from $p_{s|t}^{\theta}(\widetilde{\mathbf{G}}_s | \widetilde{\mathbf{G}}_t)$ \algcomment{Denoising}
    \STATE \( t \leftarrow s \)
\ENDFOR

\STATE \textbf{Return} $\mathbf{G}_1$
\end{algorithmic}
\end{algorithm}

\subsection{Joint Training Objective with Mutual Conditioning}\label{sec:4.5}
% \paragraph{Sampling}
To enable tight coupling between discrete graph structure and continuous 3D geometry, we introduce a shared \(\mathbf{SE}\)(3)-equivariant backbone encoder $\phi_{\theta_1}$ that unifies the feature extraction and cross-modal information fusion for both generation heads (parameterized by $\theta_2$ and $\theta_3$).
The joint training objective combines the task-specific losses of both heads, with weight hyperparameters $\lambda_{\text{disc}}$ and $\lambda_{\text{cont}} > 0$ to balance their contributions: 
% \guoqiang{$\theta$ not cover all parameters?}
\begin{equation}\label{eq:joint_loss}
\mathcal{L}_{\text{joint}}(\theta) = \lambda_{\text{disc}} \mathcal{L}_{\text{disc}}(\theta_1, \theta_3) + \lambda_{\text{cont}} \mathcal{L}_{\text{cont}}(\theta_1, \theta_2) ,
\end{equation}
where $\theta = (\theta_1,\theta_2,\theta_3)$.

\noindent\textbf{Time distortion}. Besides, we adopt a time distortion strategy, similar to the one proposed in~\cite{Qin2024DeFoG}. The key idea is to apply a non-uniform time step discretization during the sampling process, particularly emphasizing critical regions where fine-grained control is needed. 
The detailed implementation refers to \hyperref[app:appendixD]{Appendix D}. 
% \guoqiang{train?}
% \paragraph{Sampling}
% \guoqiang{ref to algorithm}

\noindent\textbf{Shared Representation Encoder}.
To enable tight coupling between discrete graph structure and continuous 3D geometry, we introduce a shared \(\textbf{SE}\)(3)-equivariant backbone encoder $\phi_{\theta_1}$ that unifies the feature extraction and cross-modal information fusion for both generation heads.
The encoder takes $\mathbf{G}_t$ and a time embedding $\tau = \text{MLP}_t(t)$ as input and generates a unified representation ${ \mathbf{h}_t = \phi_{\theta_1} \left(\mathbf{G}_t, t\right)}$, which is a condition for their individual generation process. 
This design enables the discrete graph and continuous geometry generation tasks to be mutually conditioned on the same information-rich latent representation, laying the foundation for joint generation. (See details in \hyperref[app:appendixF]{Appendix F}).

Overall, the training and sampling processes of our proposed EQUIMF are summarized in Algorithm~\ref{alg:train} and~\ref{alg:sampling}. 

\begin{table*}[t]
\centering
\caption{Results of atom stability, molecule stability, validity, and validity$\times$uniqueness.
A higher number indicates a better generation quality. The results marked with an asterisk were obtained from our own tests.}
\label{tab:qm9_metrics}
\footnotesize
\setlength{\tabcolsep}{2pt} 
\resizebox{\linewidth}{!}{
\begin{tabular}{lcccccc}
\toprule
\multirow{2}{*}{\# Metrics} & \multicolumn{4}{c}{\textbf{QM9}} & \multicolumn{2}{c}{\textbf{DRUG}}\\
\cmidrule(lr){2-5}\cmidrule(lr){6-7} 
& Atom Sta (\%) & Mol Sta (\%) & Valid (\%) & Valid \& Unique (\%) & Atom Sta (\%) & Valid (\%) \\
\midrule
Data        & 99.0 & 95.2 & 97.7 & 97.7 & 86.5 & 99.9 \\
\midrule
ENF \cite{satorras2021en}        & 85.0 &  4.9 & 40.2 & 39.4 & --   & --   \\
G-Schnet \cite{gebauer2019symmetry}    & 95.7 & 68.1 & 85.5 & 80.3 & --   & --   \\
GDM \cite{hoogeboom2022equivariant}   & 97.0 & 63.2 &  --  &  --  & 75.0 & 90.8 \\
GDM-AUG \cite{hoogeboom2022equivariant}     & 97.6 & 71.6 & 90.4 & 89.5 & 77.7 & 91.8 \\
EDM \cite{hoogeboom2022equivariant}         & 98.7 & 82.0 & 91.9 & 90.7 & 81.3 & 92.6 \\
EDM-Bridge \cite{wu2022diffusion}  & 98.8 & 84.6 & 92.0$^\ast$ & 90.7 & 82.4 & 92.8$^\ast$ \\
EQUIFM  \cite{song2023equivariant}     & {98.9 $\pm$ 0.1} & {88.0 $\pm$ 0.3} & {94.2 $\pm$ 0.2} & {93.2 $\pm$ 0.2} & {84.2} & \textbf{98.9}\\
\midrule
\textbf{EQUIMF} (ours) & \textbf{98.9 $\pm$ 0.1} & \textbf{93.0 $\pm$ 0.2} & \textbf{95.8 $\pm$ 0.4} & \textbf{95.0 $\pm$ 0.3} & \textbf{84.5} & {98.7} \\
\bottomrule
\end{tabular}
}% 闭合resizebox
\end{table*}

\subsection{Theoretical analysis}\label{sec:4.8}
In this section, we analyze the equivariance property of our proposed equivariant Meanflow (EQUIMF) generative model,  which is formally stated as follows, where the full proof is in~\hyperref[app:appendixA]{Appendix A}.
\begin{proposition}[Equivariance of EQUIMF]\label{Prop:4.4}
Assume that (i) the nodes and edges features is $\text{SE}(3)$-invariant, (ii) the average velocity field of the continuous head is $\text{SE}(3)$-equivariant, and (iii) the rate matrix of the discrete head is $\text{SE}(3)$-equivariant. Then, the whole generation process of our proposed EQUIMF is $\text{SE}(3)$-equivariant. 
% The detailed implementation and supplementary analysis are given in \hyperref[app:appendixA]{Appendix A}.
\end{proposition}
\noindent We can prove that the assumptions of the above proposition hold in our proposed EQUIMF (see~\hyperref[app:appendixA]{Appendix A}), which indicates EQUIMF keeps the equivariance inductive bias.

\section{Related Work}
\noindent\textbf{Graph Generative Models.}
Graph generative models aim to learn the distribution of complex graphs and enable sampling from this distribution, and have been widely applied to tasks such as molecular design~\cite{deepgraphsurvey2022}. From the perspective of generation paradigms, existing approaches can be broadly categorized into four classes: Autoregressive methods~\cite{molecularrnn2019,you2018graphrnn,graphaf2020} generate graphs incrementally by treating them as sequences, but often face challenges in modeling node order and permutation invariance; VAE-based method ~\cite{graphvae2018,constrainedgraphvae2018} reconstruct graph structures via latent variables, including both one-shot decoding and stepwise generation variants; GAN-based method ~\cite{de2018molgan} generate graphs or molecules through adversarial training; Normalizing flow methods ~\cite{graphnvp2019,graphnormflows2019} characterize graph distributions via invertible transformations. Diffusion model methods ~\cite{graphdiffusionsurvey2023,generativediffusiongraphs2023} that generate novel data samples from a given data distribution that employs two Markov chains.

\noindent\textbf{Discrete Diffusion and Flow Matching.} 
In recent years, the diffusion/flow paradigm has emerged as one of the mainstream approaches for graph generation. Early works often relax the adjacency matrix to a continuous space to reuse continuous diffusion frameworks~\cite{graphdiffusionsurvey2023,autoregressivediffusiongraph2024}. But this weakens the discrete structural properties of graphs and introduces inappropriate noise injection mechanisms. In contrast, discrete diffusion ~\cite{Vignac2022DiGress,Xu2024DiSco,Siraudin2024Cometh} directly defines transitions in the discrete state space, thus naturally preserving the discreteness of nodes/edges and demonstrating strong performance on various graph generation tasks. Closely related to discrete diffusion are discrete flow matching/discrete flow models ~\cite{gat2024discrete,Qin2024DeFoG,graphdf2021} based on Continuous-Time Markov Chains that their core is to characterize instantaneous transition rates using a rate matrix, with the time evolution of marginal distributions connected by the Kolmogorov equation.

\noindent\textbf{Geometric Graph Generation Models.}
Recent advances move toward joint generative modeling of molecular topology and geometry. A prominent line of work adopts continuous diffusion~\cite{graphdiffusionsurvey2023,hoogeboom2022equivariant, Chen2023EDGE} or score-based~\cite{scorebasedgraphsde2022} models or flow-based models~\cite {song2023equivariant} in Euclidean space, where atomic coordinates are gradually denoised from isotropic Gaussian noise. TTo respect physical symmetries, these models commonly incorporate equivariant neural architectures tailored to SE(3) rigid transformations .~\cite {garcia2021en,equiformer2022,equiformerv22024,se3equivariantattention2024}, such as equivariant message passing or tensor field networks, ensuring that predicted geometric updates transform consistently under rotations and translations. This design has led to substantial improvements in stability, sample validity, and physical plausibility for 3D molecule generation.

\section{Experiments}
\label{sec:experiments}

In this section, we justify the advantages of the proposed equivariant meanflow with comprehensive experiments. The experimental setup is introduced in Section \ref{sec:5.1}. Then we report and analyze the evaluation results for the 3d geometric graph generation in Section \ref{sec:5.2}. We provide the performance of controllable molecule generation that targets predefined desired properties in Section \ref{sec:5.3}. We provide detailed ablation studies in Sections \ref{sec:5.4} and \ref{sec:5.5} to further gain insight into the effects of different methods. At last, we demonstrate the high sampling efficiency in Section \ref{sec:5.6}. Other experimental details are in \hyperref[app:appendixC]{Appendix C}.

\subsection{Setup}\label{sec:5.1}
\noindent\textbf{Evaluation Tasks.}
In this study, following the prior work \cite{song2023equivariant, hoogeboom2022equivariant}, we evaluate EQUIMF on several tasks related to 3D molecular graph generation. Specifically, we assess the model's performance on the tasks of \textit{Molecular Modeling and Generation} and \textit{Conditional Molecule Generation}.

\noindent\textbf{Datasets.}
We use two commonly used datasets for our experiments.
The \textbf{QM9 dataset}~\cite{ramakrishnan2014quantum} is widely used for 3D molecular generation studies and includes 134k small organic molecules with information about various molecular properties. We use this dataset for both unconditional and conditional generation tasks. Specifically, for conditional tasks, we train models to predict chemical properties based on molecular graphs. 
We also evaluate EQUIMF on the \textbf{GEOM-DRUG dataset}~\cite{axelrod2022geom}, which is used for generating large molecular geometries. This dataset consists of large-scale molecular graphs with 3D atomic positions. It is a suitable testbed for our model’s capacity to generate molecules with realistic geometries.

\subsection{Molecular Graph Generation}\label{sec:5.2}

% \textbf{Datasets.} We choose QM9 dataset, which has been widely adopted in previous 3D molecule generation studies, for the setting of unconditional and conditional molecule generation. We also test the EquiMean on the GEOM-DRUG (Geometric Ensemble Of Molecules) dataset for generating large molecular geometries. The data configurations directly follow previous work.

\noindent\textbf{Evaluation Metrics.} Following \cite{song2023equivariant}, to assess the model’s effectiveness, we evaluate the chemical viability of the molecules it generates—this metric reflects the model’s ability to capture inherent chemical principles from the training data.
Subsequently, we gauge the quality of the predicted molecular graphs using two core stability metrics: atom stability and molecule stability. The atom stability metric calculates the fraction of atoms that exhibit the correct valence state, whereas the molecule stability metric measures the percentage of generated molecules where every atom meets stability requirements. In addition to these stability metrics, we further report validity and uniqueness: validity denotes the proportion of molecules deemed chemically valid by RDKit, and uniqueness is the percentage of distinct compounds among all generated samples.

\noindent\textbf{Baselines.}
Following the prior work~\cite{song2023equivariant,hoogeboom2022equivariant}, we compare our proposed method with existing methods on molecular generation using the methods of Equivariant models~\cite{gebauer2019symmetry} and Equivariant Normalizing Flows~\cite{garcia2021en}. Besides that, we also compared with the equivariant graph diffusion models~\cite{ho2020ddpm} and its non-equivariant variant and improved version. Finally, and most importantly, we compared it with the current SOTA method~\cite{song2023equivariant}, which is an equivariant flow matching method with Hybrid probability transport for 3d molecule generation.

\noindent\textbf{Results.} We quantitatively evaluate the performance of our method against state-of-the-art baselines on both QM9 and DRUG datasets, with results summarized in Table~\ref{tab:qm9_metrics}. Following~\cite{song2023equivariant}, we also get the above metrics from 10000 samples for each method. As shown in the table, on the QM9 dataset, our method demonstrates significant superiority across all key metrics. On the DRUG dataset, our method also delivers competitive performance. Overall, our method exhibits consistent and outstanding performance across both datasets. It not only guarantees the stability and validity of generated molecules but also enhances the diversity of outputs, thereby validating its effectiveness and competitiveness in 3D molecular generation tasks.

\begin{table}[h]
\centering
\small
\setlength{\tabcolsep}{4pt}
\caption{Mean Absolute Error for molecular property prediction.}
%\guoqiang{black the best results}
\label{tab:mae}
\begin{threeparttable}
\begin{tabular}{@{}lcccccc@{}}
\toprule
Property & $Q_{\text{min}}$ & $\Delta e$ & $\Delta v$ & $\epsilon_{\text{LUMO}}$ & $\mu$ & $C_v$ \\
         & Bohr$^2$ & eV & meV & eV & D & $\frac{\text{cal}}{\text{mol K}}$ \\
\midrule 
QM9*     & 0.10 & 64 & 39 & 36 & 0.043 & 0.00 \\
EDM$^\dagger$ & 2.76 & 655 & 536 & 584 & 1.11 & 1.101\\
EQUIFM$^\ddagger$ & 2.45 & 599 & 337 & 545 & 1.112 & 1.038 \\
EQUIMF (ours) & \textbf{2.37} & \textbf{594} & \textbf{322} & \textbf{494} & \textbf{1.080} & \textbf{1.011} \\
\bottomrule
\end{tabular}
\begin{tablenotes}[flushleft, small] % flushleft左对齐，small小字号更美观
\item $^\dagger$ \cite{hoogeboom2022equivariant}, $^\ddagger$ \cite{song2023equivariant}
% \item $^\ddagger$ \cite{song2023equivariant}
\end{tablenotes}
\end{threeparttable}
\end{table}

\noindent\textbf{Baselines.}
We select the existing method~\cite{song2023equivariant} and~\cite{hoogeboom2022equivariant} as our baselines.

\noindent\textbf{Results and Analyses.}
Table~\ref{tab:mae} reports the mean absolute error (MAE) between the target property values and the values predicted from generated molecules. Overall, our method achieves best performance across most properties, indicating improved controllability and reduced bias in conditional generation. 
This demonstrates our discrete meanflow and improving cross-domain conditioning can materially enhance the fidelity of property-controlled generation.

\subsection{Ablations on the Impacts of Equivariance}\label{sec:5.4}

To evaluate the effect of \emph{equivariant inductive biases} in our framework, we construct an ablation that differs only in whether the continuous geometric backbone and the shared backbone encoder keep equivariant coordinate updates. Specifically, we compare: \textbf{(i) EQUIMF:} our default model using an SE(3)-equivariant backbone (e.g., EGNN-style message passing) to predict the geometric MeanFlow field; \textbf{(ii) NormalMF:} a non-equivariant counterpart where the backbone is replaced by a standard graph MLP that takes the same inputs but does not guarantee equivariance (i.e., coordinates are updated without the SE(3)-equivariant constraint). We evaluate on QM9 using \emph{Atom Stable (\%)} and \emph{Mol Stable (\%)} following common practice. Table~\ref{tab:ablation_1} shows that equivariant inductive biases consistently improve both stability metrics. These results validate that incorporating SE(3)-equivariant inductive bias into the continuous MeanFlow backbone is a key factor for stable and physically consistent molecule generation.
\begin{table}[H]
  \caption{Ablation results of our models trained with/without equivariance on the QM9 dataset.}
  \label{tab:ablation_1}
  \begin{center}
    \begin{small}
      \begin{sc}
        \begin{tabular}{lcccr}
        \toprule
        Method & Atom Stable (\%) & Mol Stable (\%) \\
        \midrule
        $\text{EQUIMF}$ & $\textbf{98.9} \pm 0.1$ & $\textbf{93.0} \pm 0.2$ \\
        $\text{NormalMF}$ & $97.3 \pm 0.1$ & $88.1 \pm 0.1$ \\
        \bottomrule
        \end{tabular}
      \end{sc}
    \end{small}
  \end{center}
  \vskip -0.1in
\end{table}

\subsection{Ablations on the Impacts of Mutual Conditioning}\label{sec:5.5}
To validate the necessity of bidirectional mutual conditioning between discrete topology and continuous geometry, we consider four different approaches to model the distribution $p_{1|t} (\mathbf{G}_1 | \mathbf{G}_t)$: 
\begin{align*}
    & \text{P1}: \quad  p_{1|t}^{\theta} (\mathbf{G}_1 | \mathbf{G}_t) := p_{1|t}^{\theta_1, \theta_3} (\widetilde{\mathbf{G}}_1 | \mathbf{G}_t) p_{1|t}^{\theta_1, \theta_2} (\mathbf{R}_1 | \mathbf{G}_t), \\
    & \text{P2}: \quad  p_{1|t}^{\theta} (\mathbf{G}_1 | \mathbf{G}_t) := p_{1|t}^{\theta_3} (\widetilde{\mathbf{G}}_1 | \widetilde{\mathbf{G}}_t) p_{1|t}^{\theta_1, \theta_2} (\mathbf{R}_1 | \mathbf{G}_t), \\
    & \text{P3}: \quad  p_{1|t}^{\theta} (\mathbf{G}_1 | \mathbf{G}_t) := p_{1|t}^{\theta_1, \theta_3} (\widetilde{\mathbf{G}}_1 | \mathbf{G}_t) p_{1|t}^{\theta_2} (\mathbf{R}_1 | \mathbf{R}_t), \\
    & \text{P4}: \quad  p_{1|t}^{\theta} (\mathbf{G}_1 | \mathbf{G}_t) := p_{1|t}^{\theta_3} (\widetilde{\mathbf{G}}_1 | \widetilde{\mathbf{G}}_t) p_{1|t}^{\theta_2} (\mathbf{R}_1 | \mathbf{R}_t).
\end{align*}
As illustrated in Table~\ref{tab:ablation_2}, the results validate that mutual conditioning is critical for generating stable molecules: geometry-aware topological updates reduce chemically invalid edges, while structure-guided geometric updates prevent physically implausible conformations.
\begin{table}[H]
  \caption{Ablation results of meanflow models trained with/without mutual conditioning on the QM9 dataset.}
  \label{tab:ablation_2}
  \begin{center}
    \begin{small}
      \begin{sc}
        \begin{tabular}{lcccr}
        \toprule
        Method & Atom Stable (\%) & Mol Stable (\%) \\
        \midrule
        $\text{P1}$ (ours) & $\textbf{98.9} \pm 0.1$ & $\textbf{93.0} \pm 0.2$ \\
        $\text{P2}$ & $85.1 \pm 0.1$ & $79.4 \pm 0.1$ \\
        $\text{P3}$ & $83.6 \pm 0.1$ & $77.2 \pm 0.1$ \\
        $\text{P4}$ & $33.7 \pm 0.1$ & $28.5 \pm 0.1$ \\
        \bottomrule
        \end{tabular}
      \end{sc}
    \end{small}
  \end{center}
  \vskip -0.1in
\end{table}

\subsection{Sampling Efficiency}\label{sec:5.6}
% \guoqiang{need to revise, figure, expression, 2 $\times$ accelerating?}
We evaluate the convergence efficiency and generation quality of our method by tracking the molecular stability during the sampling process, with EquiFM as the baseline. The variation curve of stability with sampling steps is shown in Figure \ref{fig:Sampling}. At the 0.95 stability threshold, our method requires only half the number of steps of the baseline, representing a nearly 2$\times$ improvement in efficiency. Besides, we get higher final stability.
\begin{figure}[htbp]\label{fig:Sampling}
    \centering
    \includegraphics[width=0.5\textwidth]{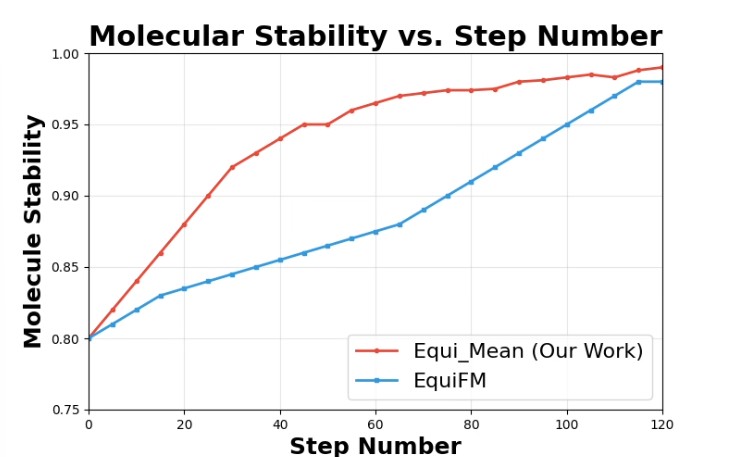} 
    \caption{Molecular Stability vs. Step Number}
    \label{fig:Sampling}
\end{figure}

\section{Conclusion and Discussion}
% We propose EQUIMF, a unified SE(3)-equivariant generative framework that jointly models discrete and continuous domains through synchronized MeanFlow dynamics. 
% We couple the modeling of discrete structural components and continuous geometric properties by leveraging synchronized MeanFlow dynamics, and we ground theoretically to ensure equivariant graph distribution modeling. 
% Our results on a series of benchmarks show that Equivariant MeanFlow consistently improves generation quality, physical validity, and sampling efficiency over prior flow-matching and diffusion models. 
We present EQUIMF, a unified SE(3)-equivariant generative framework that models discrete and continuous domains jointly via synchronized MeanFlow dynamics. 
By coupling discrete structural and continuous geometric modeling and establishing theoretical guarantees for equivariant graph distribution learning, EQUIMF outperforms existing flow-matching and diffusion models across benchmarks in generation quality, physical validity, and sampling efficiency.
Besides, our proposed discrete MeanFlow can be used in other discrete domains, e.g., text generation. EQUIMF achieves highly efficient few-step sampling but does not fully utilize the core MeanFlow merit of single-step sampling. 
Its current design balances efficiency and quality through a few-step evolution, leaving one-step discrete-continuous generation unexplored.
% \guoqiang{new discrete meanflow}

\section*{Impact Statement}
``This paper presents work whose goal is to advance the field of Machine
Learning. There are many potential societal consequences of our work, none
which we feel must be specifically highlighted here.''

\bibliography{refs}

\begin{thebibliography}{10}

\bibitem{deepgraphsurvey2022}
Yanqiao Zhu, Yuanqi Du, Yinkai Wang, Yichen Xu, Jieyu Zhang, Qiang Liu, and Shu Wu.
\newblock A survey on deep graph generation: Methods and applications.
\newblock In {\em First Learning on Graphs Conference (LoG 2022)}, 2022.
\newblock Accepted by LoG 2022.

\bibitem{schutt2017quantum}
Kristof~T. Sch{\"u}tt, Farhad Arbabzadah, Stefan Chmiela, Klaus~R. M{\"u}ller, and Alexandre Tkatchenko.
\newblock Quantum-chemical insights from deep tensor neural networks.
\newblock {\em Nature Communications}, 8:13890, 2017.

\bibitem{autoregressivediffusiongraph2024}
Lingkai Kong, Jiaming Cui, Haotian Sun, Yuchen Zhuang, B.~Aditya Prakash, and Chao Zhang.
\newblock Autoregressive diffusion model for graph generation.
\newblock {\em arXiv preprint arXiv:2307.08849}, 2024.

\bibitem{ma2024exploring}
N.~Ma, M.~Goldstein, M.~S. Albergo, N.~M. Boffi, E.~Vanden-Eijnden, and S.~Xie.
\newblock Exploring flow and diffusion-based generative models with scalable interpolant transformers.
\newblock In {\em European Conference on Computer Vision (ECCV)}, 2024.

\bibitem{molecularrnn2019}
Mariya Popova, Mykhailo Shvets, Junier Oliva, and Olexandr Isayev.
\newblock Molecularrnn: Generating realistic molecular graphs with optimized properties.
\newblock {\em arXiv preprint arXiv:1905.13372}, 2019.

\bibitem{Vignac2022DiGress}
Cl{\'e}ment Vignac, Ireneusz Krawczuk, Alexandre Siraudin, Bohan Wang, Volkan Cevher, and Pascal Frossard.
\newblock Digress: Discrete denoising diffusion for graph generation.
\newblock In {\em International Conference on Machine Learning (ICML)}, 2022.

\bibitem{Xu2024DiSco}
Z.~Xu, R.~Qiu, Y.~Chen, H.~Chen, X.~Fan, M.~Pan, Z.~Zeng, M.~Das, and H.~Tong.
\newblock Discrete-state continuous-time diffusion for graph generation.
\newblock In {\em Advances in Neural Information Processing Systems (NeurIPS)}, 2024.

\bibitem{Siraudin2024Cometh}
Alexandre Siraudin, Fragkiskos~D. Malliaros, and Christopher Morris.
\newblock Cometh: A continuous-time discrete-state graph diffusion model.
\newblock {\em arXiv preprint arXiv:2406.06449}, 2024.

\bibitem{campbell2024generative}
A.~Campbell, J.~Yim, R.~Barzilay, T.~Rainforth, and T.~Jaakkola.
\newblock Generative flows on discrete state-spaces: Enabling multimodal flows with applications to protein co-design.
\newblock In {\em International Conference on Machine Learning (ICML)}, 2024.

\bibitem{Qin2024DeFoG}
Yiming Qin, Manuel Madeira, Dorina Thanou, and Pascal Frossard.
\newblock Defog: Discrete flow matching for graph generation.
\newblock In {\em Proceedings of the 42nd International Conference on Machine Learning (ICML)}, 2025.

\bibitem{song2023equivariant}
Yuxuan Song, Jingjing Gong, Minkai Xu, Ziyao Cao, Yanyan Lan, Stefano Ermon, Hao Zhou, and Wei-Ying Ma.
\newblock Equivariant flow matching with hybrid probability transport.
\newblock {\em arXiv preprint arXiv:2312.07168}, 2023.
\newblock NeurIPS 2023.

\bibitem{hoogeboom2022equivariant}
Emiel Hoogeboom, Victor Garcia~Satorras, Cl{\'e}ment Vignac, and Max Welling.
\newblock Equivariant diffusion for molecule generation in 3d.
\newblock {\em arXiv preprint arXiv:2203.17003}, 2022.
\newblock Accepted at International Conference on Machine Learning (ICML 2022).

\bibitem{liu2023flow}
Xingchao Liu, Chengyue Gong, and Qiang Liu.
\newblock Flow straight and fast: Learning to generate and transfer data with rectified flow.
\newblock In {\em International Conference on Learning Representations (ICLR)}, 2023.

\bibitem{lipman2022flow}
Yaron Lipman, Ricky T.~Q. Chen, Heli Ben-Hamu, Maximilian Nickel, and Minh Le.
\newblock Flow matching for generative modeling.
\newblock {\em arXiv preprint arXiv:2210.02747}, 2022.

\bibitem{albergo2023stochastic}
Michael~Samuel Albergo and Eric Vanden-Eijnden.
\newblock Building normalizing flows with stochastic interpolants.
\newblock In {\em International Conference on Learning Representations (ICLR)}, 2023.

\bibitem{garcia2021en}
Victor Garcia~Satorras, Emiel Hoogeboom, and Max Welling.
\newblock E(n) equivariant graph neural networks.
\newblock {\em arXiv preprint arXiv:2102.09844}, 2021.

\bibitem{geng2025mean}
Zhengyang Geng, Mingyang Deng, Xingjian Bai, J.~Zico Kolter, and Kaiming He.
\newblock Mean flows for one-step generative modeling, 2025.
\newblock Tech report.

\bibitem{satorras2021en}
Victor Garcia~Satorras, Emiel Hoogeboom, Fabian~B. Fuchs, Ingmar Posner, and Max Welling.
\newblock E(n) equivariant normalizing flows.
\newblock {\em Advances in Neural Information Processing Systems}, 34:20183--20195, 2021.
\newblock Accepted at NeurIPS 2021.

\bibitem{gebauer2019symmetry}
Niklas W.~A. Gebauer, Michael Gastegger, and Kristof~T. Schütt.
\newblock Symmetry-adapted generation of 3d point sets for the targeted discovery of molecules, 2019.

\bibitem{wu2022diffusion}
Lemeng Wu, Chengyue Gong, Xingchao Liu, Mao Ye, and Qiang Liu.
\newblock Diffusion-based molecule generation with informative prior bridges.
\newblock {\em arXiv preprint arXiv:2209.00865}, 2022.

\bibitem{you2018graphrnn}
Jiaxuan You, Rex Ying, Xiang Ren, William~L. Hamilton, and Jure Leskovec.
\newblock {GraphRNN}: Generating realistic graphs with deep autoregressive models.
\newblock In {\em International Conference on Machine Learning (ICML)}, 2018.

\bibitem{graphaf2020}
Chence Shi, Minkai Xu, Zhaocheng Zhu, Weinan Zhang, Ming Zhang, and Jian Tang.
\newblock Graphaf: a flow-based autoregressive model for molecular graph generation.
\newblock In {\em International Conference on Learning Representations (ICLR 2020)}, 2020.

\bibitem{graphvae2018}
Martin Simonovsky and Nikos Komodakis.
\newblock Graphvae: Towards generation of small graphs using variational autoencoders.
\newblock {\em arXiv preprint arXiv:1802.03480}, 2018.

\bibitem{constrainedgraphvae2018}
Qi~Liu, Miltiadis Allamanis, Marc Brockschmidt, and Alexander~L. Gaunt.
\newblock Constrained graph variational autoencoders for molecule design.
\newblock {\em arXiv preprint arXiv:1805.09076}, 2018.

\bibitem{de2018molgan}
Nicola De~Cao and Thomas Kipf.
\newblock An implicit generative model for small molecular graphs.
\newblock In {\em International Conference on Machine Learning (ICML) Workshops}, 2018.

\bibitem{graphnvp2019}
Kaushalya Madhawa, Katushiko Ishiguro, Kosuke Nakago, and Motoki Abe.
\newblock Graphnvp: An invertible flow model for generating molecular graphs.
\newblock {\em arXiv preprint arXiv:1905.11600}, 2019.

\bibitem{graphnormflows2019}
Jenny Liu, Aviral Kumar, Jimmy Ba, Jamie Kiros, and Kevin Swersky.
\newblock Graph normalizing flows.
\newblock {\em arXiv preprint arXiv:1905.13177}, 2019.

\bibitem{graphdiffusionsurvey2023}
Mengchun Zhang, Maryam Qamar, Taegoo Kang, Yuna Jung, Chenshuang Zhang, Sung-Ho Bae, and Chaoning Zhang.
\newblock A survey on graph diffusion models: Generative ai in science for molecule, protein and material.
\newblock {\em arXiv preprint arXiv:2304.01565}, 2023.

\bibitem{generativediffusiongraphs2023}
Chengyi Liu, Wenqi Fan, Yunqing Liu, Jiatong Li, Hang Li, Hui Liu, Jiliang Tang, and Qing Li.
\newblock Generative diffusion models on graphs: Methods and applications.
\newblock In {\em Proceedings of the 32nd International Joint Conference on Artificial Intelligence (IJCAI 2023)}, 2023.
\newblock Accepted by IJCAI 2023.

\bibitem{gat2024discrete}
I.~Gat, T.~Remez, N.~Shaul, F.~Kreuk, R.~T. Chen, G.~Synnayeve, Y.~Adi, and Y.~Lipman.
\newblock Discrete flow matching.
\newblock In {\em Advances in Neural Information Processing Systems (NeurIPS)}, 2024.

\bibitem{graphdf2021}
Youzhi Luo, Keqiang Yan, and Shuiwang Ji.
\newblock Graphdf: A discrete flow model for molecular graph generation.
\newblock In {\em The 38th International Conference on Machine Learning (ICML 2021)}, 2021.
\newblock Accepted by ICML 2021.

\bibitem{Chen2023EDGE}
X.~Chen, J.~He, X.~Han, and L.-P. Liu.
\newblock Efficient and degree-guided graph generation via discrete diffusion modeling.
\newblock In {\em International Conference on Machine Learning (ICML)}, 2023.

\bibitem{scorebasedgraphsde2022}
Jaehyeong Jo, Seul Lee, and Sung~Ju Hwang.
\newblock Score-based generative modeling of graphs via the system of stochastic differential equations.
\newblock In {\em The 39th International Conference on Machine Learning (ICML 2022)}, 2022.

\bibitem{equiformer2022}
Yi-Lun Liao and Tess Smidt.
\newblock Equiformer: Equivariant graph attention transformer for 3d atomistic graphs.
\newblock {\em arXiv preprint arXiv:2206.11990}, 2022.

\bibitem{equiformerv22024}
Yi-Lun Liao, Brandon Wood, Abhishek Das, and Tess Smidt.
\newblock Equiformerv2: Improved equivariant transformer for scaling to higher-degree representations.
\newblock In {\em International Conference on Learning Representations (ICLR 2024)}, 2024.

\bibitem{se3equivariantattention2024}
Evangelos Chatzipantazis, Stefanos Pertigkiozoglou, Edgar Dobriban, and Kostas Daniilidis.
\newblock Se(3)-equivariant attention networks for shape reconstruction in function space.
\newblock {\em arXiv preprint arXiv:2204.02394}, 2024.

\bibitem{ramakrishnan2014quantum}
R.~Ramakrishnan, P.O. Dral, M.~Rupp, and O.A. Von~Lilienfeld.
\newblock Quantum chemistry structures and properties of 134 kilo molecules.
\newblock {\em Scientific Data}, 1:140022, 2014.

\bibitem{axelrod2022geom}
Simon Axelrod and Rafael Gómez-Bombarelli.
\newblock {GEOM, energy-annotated molecular conformations for property prediction and molecular generation}.
\newblock {\em Scientific Data}, 2022.

\bibitem{ho2020ddpm}
Jonathan Ho, Ajay Jain, and Pieter Abbeel.
\newblock Denoising diffusion probabilistic models.
\newblock In {\em Advances in Neural Information Processing Systems (NeurIPS)}, 2020.

\bibitem{liao2019gran}
R.~Liao, Y.~Li, Y.~Song, S.~Wang, W.~Hamilton, D.~K. Duvenaud, R.~Urtasun, and R.~Zemel.
\newblock Efficient graph generation with graph recurrent attention networks.
\newblock In {\em Advances in Neural Information Processing Systems (NeurIPS)}, 2019.

\bibitem{Martinkus2022SPECTRE}
Karolis Martinkus, Andreas Loukas, Nicolas Perraudin, and Roger Wattenhofer.
\newblock {SPECTRE}: Spectral conditioning helps to overcome the expressivity limits of one-shot graph generators.
\newblock In {\em International Conference on Machine Learning (ICML)}, 2022.

\bibitem{Diamant2023Bandwidth}
N.~L. Diamant, A.~M. Tseng, K.~V. Chuang, T.~Biancalani, and G.~Scalia.
\newblock Improving graph generation by restricting graph bandwidth.
\newblock In {\em International Conference on Machine Learning (ICML)}, 2023.

\bibitem{Dai2020ScalableSparseGraphs}
H.~Dai, A.~Nazi, Y.~Li, B.~Dai, and D.~Schuurmans.
\newblock Scalable deep generative modeling for sparse graphs.
\newblock In {\em International Conference on Machine Learning (ICML)}, 2020.

\bibitem{Goyal2020GraphGen}
Nikhil Goyal, Harshit~V. Jain, and Sayan Ranu.
\newblock Graphgen: A scalable approach to domain-agnostic labeled graph generation.
\newblock In {\em Proceedings of The Web Conference (WWW)}, 2020.

\bibitem{Bergmeister2023ILE}
Andreas Bergmeister, Karolis Martinkus, Nicolas Perraudin, and Roger Wattenhofer.
\newblock Efficient and scalable graph generation through iterative local expansion.
\newblock In {\em International Conference on Learning Representations (ICLR)}, 2023.

\bibitem{Jo2024DiffusionMixture}
J.~Jo, D.~Kim, and S.~J. Hwang.
\newblock Graph generation with diffusion mixture.
\newblock In {\em International Conference on Machine Learning (ICML)}, 2024.

\bibitem{Eijkelboom2024VFM}
Floris Eijkelboom, Gregory Bartosh, Christian Andersson~Naesseth, Max Welling, and Jan-Willem van~de Meent.
\newblock Variational flow matching for graph generation.
\newblock In {\em Advances in Neural Information Processing Systems (NeurIPS)}, 2024.

\bibitem{spectre2022}
Karolis Martinkus, Andreas Loukas, Nathana\"el Perraudin, and Roger Wattenhofer.
\newblock Spectre: Spectral conditioning helps to overcome the expressivity limits of one-shot graph generators.
\newblock In {\em The 39th International Conference on Machine Learning (ICML 2022)}, page 21 pages, 2022.
\newblock 10 figures.

\bibitem{iterative2024}
Andreas Bergmeister, Karolis Martinkus, Nathana\"el Perraudin, and Roger Wattenhofer.
\newblock Efficient and scalable graph generation through iterative local expansion.
\newblock In {\em International Conference on Learning Representations (ICLR 2024)}, 2024.
\newblock Published as a conference paper.

\end{thebibliography}
\bibliographystyle{unsrt}

%%%%%%%%%%%%%%%%%%%%%%%%%%%%%%%%%%%%%%%%%%%%%%%%%%%%%%%%%%%%%%%%%%%%%%%%%%%%%%%
%%%%%%%%%%%%%%%%%%%%%%%%%%%%%%%%%%%%%%%%%%%%%%%%%%%%%%%%%%%%%%%%%%%%%%%%%%%%%%%
% APPENDIX
%%%%%%%%%%%%%%%%%%%%%%%%%%%%%%%%%%%%%%%%%%%%%%%%%%%%%%%%%%%%%%%%%%%%%%%%%%%%%%%
%%%%%%%%%%%%%%%%%%%%%%%%%%%%%%%%%%%%%%%%%%%%%%%%%%%%%%%%%%%%%%%%%%%%%%%%%%%%%%%
\newpage
\appendix
\onecolumn

\section{Formal Proof of Propositions}\label{app:appendixA}
% \guoqiang{notations are not consistent with the above paper}

% \begin{align*}
%     p_{1|t}^{\theta} (\mathbf{G}_1 | \mathbf{G}_t) = p_{1|t}^{\theta_1, \theta_3} (\widetilde{\mathbf{G}}_1 | \mathbf{G}_t) p_{1|t}^{\theta_1, \theta_2} (\mathbf{R}_1 | \mathbf{G}_t) \\
%     p_{1|t}^{\theta} (\mathbf{G}_1 | \mathbf{G}_t) = p_{1|t}^{\theta_3} (\widetilde{\mathbf{G}}_1 | \widetilde{\mathbf{G}}_t) p_{1|t}^{\theta_1, \theta_2} (\mathbf{R}_1 | \mathbf{G}_t) \\
%     p_{1|t}^{\theta} (\mathbf{G}_1 | \mathbf{G}_t) = p_{1|t}^{\theta_1, \theta_3} (\widetilde{\mathbf{G}}_1 | \mathbf{G}_t) p_{1|t}^{\theta_2} (\mathbf{R}_1 | \mathbf{R}_t) \\
%     p_{1|t}^{\theta} (\mathbf{G}_1 | \mathbf{G}_t) = p_{1|t}^{\theta_3} (\widetilde{\mathbf{G}}_1 | \widetilde{\mathbf{G}}_t) p_{1|t}^{\theta_2} (\mathbf{R}_1 | \mathbf{R}_t)
% \end{align*}

Let $g=(\mathbf{Q},\mathbf{a})\in \textbf{SE}(3)$ where $\mathbf{Q}\in \textbf{O}(3)$ and $\mathbf{a}\in\mathbb{R}^3$.
For atomic coordinates $\mathbf{R}=[\mathrm{r}_1,\dots,\mathrm{r}_N]^\top\in\mathbb{R}^{N\times 3}$,
we define the rigid action
\begin{equation}
g\cdot \mathbf{r}_i = \mathbf{Q}\mathbf{r}_i+\mathbf{a},
\qquad
g\cdot \mathbf{R} = \mathbf{R}\mathbf{Q}^\top + \mathbf{1}\mathbf{a}^\top.
\label{eq:se3_action_matrix}
\end{equation}
For discrete node/edge states $(\mathbf{X},\mathbf{E})$, we treat them as $\textbf{SE}(3)$-invariant:
\begin{equation}
g\cdot (\mathbf{X},\mathbf{E},\mathbf{R}) \triangleq (\mathbf{X},\mathbf{E},\,g\cdot \mathbf{R}).
\label{eq:se3_action_graph}
\end{equation}

For any $i\neq j$ and any $g=(\mathbf{Q},\mathbf{a})\in\textbf{SE}(3)$,
\begin{equation}
\|(g\cdot \mathbf{r}_i)-(g\cdot \mathbf{r}_j)\|_2^2
=
\|\mathbf{Q}(\mathbf{r}_i-\mathbf{r}_j)\|_2^2
=
(\mathbf{r}_i-\mathbf{r}_j)^\top \mathbf{Q}^\top \mathbf{Q} (\mathbf{r}_i-\mathbf{r}_j)
=
\|\mathbf{r}_i-\mathbf{r}_j\|_2^2.
\label{eq:distance_invariant}
\end{equation}
Hence all pairwise squared distances are $\textbf{SE}(3)$-invariant.

\subsection{Invariant Discrete Nodes and edges features}

\begin{proposition}[Invariant discrete node and edge features]
\label{prop:inv_discrete}
Let $\mathbf{G}_t=(\mathbf{X}_t,\mathbf{E}_t,\mathbf{R}_t)$ be a noisy molecular graph at time $t$.
Assume node states $\mathbf{X}_t$ and edge states $\mathbf{E}_t$ represent discrete types (e.g., atom/bond categories). Then for any $g\in\textbf{SE}(3)$,
\begin{equation}
(\mathbf{X}_t',\mathbf{E}_t') = (\mathbf{X}_t,\mathbf{E}_t)
\quad\text{where}\quad
\mathbf{G}_t' = g\cdot \mathbf{G}_t = (\mathbf{X}_t',\mathbf{E}_t',g\cdot \mathbf{R}_t).
\end{equation}
That is, discrete node/edge features are $\mathbf{SE}(3)$-invariant.
\end{proposition}

\begin{proof}
By definition, discrete node features \(\mathbf{X}_t\) and edge features \(\mathbf{E}_t\) encode the \(\textit{intrinsic chemical properties}\). These properties are independent of the global coordinate frame of \(\mathbf{R}_t\), as they do not depend on the spatial position or orientation of the molecule. 
For any \(\textbf{SE}(3)\) transformation \(g = (\mathbf{Q}, \mathbf{a}) \in \textbf{SE}(3)\) (consisting of a rotation \(\mathbf{Q} \in \textbf{SO}(3)\) and a translation \(\mathbf{a} \in \mathbb{R}^3\)), the action of \(g\) only affects the 3D atomic coordinates \(\mathbf{R}_t\) (transforming them to \(g \cdot \mathbf{R}_t = \mathbf{Q} \mathbf{R}_t + \mathbf{a}\)). Since \(\mathbf{X}_t\) and \(\mathbf{E}_t\) are independent of \(\mathbf{R}_t\), applying \(g\) does not alter the discrete types of nodes or edges. Thus:
\[
\mathbf{X}_t' = \mathbf{X}_t,\quad \mathbf{E}_t' = \mathbf{E}_t
\]
This implies \((\mathbf{X}_t,\mathbf{E}_t) = (\mathbf{X}_t', \mathbf{E}_t')\), confirming that discrete node and edge features are \(\textbf{SE}(3)\)-invariant.

\end{proof}

\subsection{Equivariance of Average Velocity}
\begin{proposition}[Equivariance of average velocity]
\label{prop:avg_vel_equiv}
Let $t<s$ and define the (ground-truth) average velocity field
\begin{equation}
{u}_{t\to s}(\mathbf{R}_t)
\triangleq
\frac{\mathbf{R}_s-\mathbf{R}_t}{s-t}\in\mathbb{R}^{N\times 3}.
\label{eq:avg_vel_def}
\end{equation}
Then for any $g=(\mathbf{Q},\mathbf{a})\in\textbf{SE}(3)$,
\begin{equation}
{u}_{t\to s}(g\cdot \mathbf{R}_t)={u}_{t\to s}(\mathbf{R}_t)\,\mathrm{Q}^\top,
\label{eq:avg_vel_equivariance}
\end{equation}
i.e., the average velocity is $\mathbf{SE}(3)$-equivariant.
Moreover, if the continuous MeanFlow head is implemented by an
$\textbf{SE}(3)$-equivariant network (e.g., EGNN) producing
${u}_{t\to s} = \phi_{\theta_2}(\cdot)$, then
\begin{equation}
\hat u_{t\to s}(g\cdot \mathbf{G}_t)=\hat u_{t\to s}(\mathbf{G}_t)\,\mathbf{Q}^\top .
\label{eq:learned_avg_vel_equiv}
\end{equation}
\end{proposition}

\begin{proof}
Using Eq.~\eqref{eq:se3_action_matrix}, we have
\[
g\cdot \mathbf{R}_s - g\cdot \mathbf{R}_t
=
(\mathbf{R}_s\mathbf{Q}^\top+\mathbf{1}\mathbf{a}^\top) - (\mathbf{R}_t\mathbf{Q}^\top+\mathbf{1}\mathbf{a}^\top)
=
(\mathbf{R}_s-\mathbf{R}_t)\mathbf{Q}^\top.
\]
Dividing by $(s-t)$ yields
\[
{u}_{t\to s}(g\cdot \mathbf{R}_t)
=
\frac{g\cdot \mathbf{R}_s - g\cdot \mathbf{R}_t}{s-t}
=
\frac{(\mathbf{R}_s-\mathrm{R}_t)\mathbf{Q}^\top}{s-t}
=
{u}_{t\to s}(\mathbf{R}_t)\,\mathbf{Q}^\top,
\]
which proves Eq.~\eqref{eq:avg_vel_equivariance}.
The learned equivariance Eq.~\eqref{eq:learned_avg_vel_equiv} follows directly
from the assumed $\mathrm{SE}(3)$-equivariance of the backbone and the head
(e.g., EGNN layers preserve the rigid-motion action while operating only on
invariants such as Eq.~\eqref{eq:distance_invariant}).
\end{proof}

\subsection{Invariance of rate matrices}
\begin{proposition}[Invariance of rate matrices]
\label{prop:rate_matrix_equiv}
Let the discrete head parameterize node/edge CTMC rate matrices
$\mathbf{\bar{K}}_t^{\mathbf{X}}$ and $\mathbf{\bar{K}}_t^{\mathbf{E}}$ conditioned on $\mathbf{G}_t=(\mathbf{X}_t,\mathbf{E}_t,\mathbf{R}_t)$.
Assume the rate predictors depend on coordinates only through
$\textbf{SE}(3)$-invariant quantities (e.g., pairwise distances
$\|\mathbf{r}_{t,i}-\mathbf{r}_{t,j}\|_2^2$, or other rigid invariants) and on $(\mathbf{X}_t,\mathbf{E}_t)$
through their discrete values. Then for any $g\in\textbf{SE}(3)$,
\begin{equation}
\mathbf{\bar{K}}_t^{\mathbf{X}}(g\cdot \mathbf{G}_t)=\mathbf{\bar{K}}_t^{\mathbf{X}}(\mathbf{G}_t),
\qquad
\mathbf{\bar{K}}_t^{\mathbf{E}}(g\cdot \mathbf{G}_t)=\mathbf{\bar{K}}_t^{\mathbf{E}}(\mathbf{G}_t).
\label{eq:rate_invariant}
\end{equation}
Consequently, the induced discrete transition kernel
$p_\theta^{\mathrm{disc}}(\mathbf{X}_s,\mathbf{E}_s| \mathbf{G}_t,t,\Delta)$ is $\textbf{SE}(3)$-invariant.
\end{proposition}

\begin{proof}
Under $g=(\mathbf{Q},\mathbf{a})$, the discrete inputs $(\mathbf{X}_t,\mathbf{E}_t)$ are unchanged
(Proposition~\ref{prop:inv_discrete}).
By Eq.~\eqref{eq:distance_invariant}, all pairwise squared distances (and any
function thereof) are unchanged.
Hence every scalar input used by the rate predictors is identical under
$(\mathbf{X}_t,\mathbf{E}_t,\mathbf{R}_t)$ and $(\mathbf{X}_t,\mathbf{E}_t,g\cdot \mathbf{R}_t)$, implying the predicted rate
matrices are identical, i.e., Eq.~\eqref{eq:rate_invariant}. Since a CTMC transition kernel over discrete states is fully determined by its
rate matrix (e.g., via matrix exponential or Euler discretization), the
resulting conditional distribution over $(\mathbf{X}_s,\mathbf{E}_s)$ is also unchanged under $g$,
thus $\textbf{SE}(3)$-invariant.
\end{proof}

\subsection{Equivariance of EQUIMF: Proposition~\ref{Prop:4.4}}

\begin{proposition}[Equivariance of EQUIMF]
\label{prop:meanflow_equiv}
Consider one coupled MeanFlow step along a time bridge $(t,s)$:
(i) sample/update $(\mathbf{X}_s,\mathbf{E}_s)$ from the discrete kernel determined by
$\mathbf{\bar{K}}_t^{\mathbf{X}},\mathbf{\bar{K}}_t^{\mathbf{E}}$, and (ii) update coordinates by the average-velocity field
\begin{equation}
\mathbf{R}_s = \mathbf{R}_t + (s-t)\,{u}_{t\to s}(\mathbf{G}_t).
\label{eq:meanflow_cont_update}
\end{equation}
Assume Proposition~\ref{prop:inv_discrete}, \ref{prop:avg_vel_equiv},
and \ref{prop:rate_matrix_equiv} hold.
Then the overall one-step transition operator is $\textbf{SE}(3)$-equivariant:
for any $g\in\textbf{SE}(3)$, if $\mathbf{G}_t' = g\cdot \mathbf{G}_t$, then the next state
$\mathbf{G}_s'$ produced by the same transition satisfies
\begin{equation}
\mathbf{G}_s' = g\cdot \mathbf{G}_s .
\label{eq:coupled_equiv_step}
\end{equation}
\end{proposition}

\begin{proof}
\emph{(Discrete part).}
By Proposition~\ref{prop:rate_matrix_equiv}, the rate matrices (hence the
discrete kernel) are invariant under $g$; therefore sampling $(\mathbf{X}_s,\mathbf{E}_s)$ from
$G_t$ or from $\mathbf{G}_t'=g\cdot \mathbf{G}_t$ yields the same distribution.
Since $(\mathbf{X},\mathbf{E})$ are invariant variables, we have $(\mathbf{X}_s',\mathbf{E}_s')=(\mathbf{X}_s,\mathbf{E}_s)$.

\emph{(Continuous part).}
Let $g=(\mathbf{Q},\mathbf{a})$.
Using Eq.~\eqref{eq:meanflow_cont_update} and Proposition~\ref{prop:avg_vel_equiv},
\[
\mathbf{R}_s'
=
g\cdot \mathbf{R}_t + (s-t)\,{u}_{t\to s}(\mathbf{G}_t') 
=
(\mathbf{R}_t\mathbf{Q}^\top+\mathbf{1}\mathbf{a}^\top) + (s-t)\,{u}_{t\to s}(\mathbf{G}_t)\,\mathbf{Q}^\top.
\]
Factorizing $\mathbf{Q}^\top$ gives
\[
\mathbf{R}_s'
=
\big(\mathbf{R}_t + (s-t){u}_{t\to s}(\mathbf{G}_t)\big)\mathbf{Q}^\top + \mathbf{1}\mathbf{a}^\top
=
\mathbf{R}_s \mathbf{Q}^\top + \mathbf{1}\mathbf{a}^\top
=
g\cdot \mathbf{R}_s.
\]
Combining with $(\mathbf{X}_s',\mathbf{E}_s')=(\mathbf{X}_s,\mathbf{E}_s)$ yields $\mathbf{G}_s'=(\mathbf{X}_s',\mathbf{E}_s',\mathbf{R}_s')=g\cdot \mathbf{G}_s$,
proving Eq.~\eqref{eq:coupled_equiv_step}.
\end{proof}

\section{Synthetic Graph Generation Performance}\label{app:appendixC}
\begin{table*}[]
% \begin{center}
\centering
\captionof{table}{Graph generation performance on the synthetic datasets: Planar, Tree and SBM. 
Results are mean $\pm$ std over five runs (40 graphs/run).}
\label{tab:synthetic}
\scriptsize
\setlength{\tabcolsep}{6pt}
\begin{tabular}{l l | cc | cc | cc}
\toprule
Model & Class & \multicolumn{2}{c|}{Planar} & \multicolumn{2}{c|}{Tree} & \multicolumn{2}{c}{SBM} \\
\cmidrule(lr){3-4}\cmidrule(lr){5-6}\cmidrule(lr){7-8}
 &  & V.U.N.~$\uparrow$ & Ratio~$\downarrow$ & V.U.N.~$\uparrow$ & Ratio~$\downarrow$ & V.U.N.~$\uparrow$ & Ratio~$\downarrow$ \\
\midrule
Train set & --- & 100 & 1.0 & 100 & 1.0 & 85.9 & 1.0 \\
\midrule
GraphRNN (~\cite{you2018graphrnn}) & Autoregressive & 0.0 & 490.2 & 0.0 & 607.0 & 5.0 & 14.7 \\
GRAN (~\cite{liao2019gran})    & Autoregressive & 0.0 & 2.0   & 0.0 & 607.0 & 25.0 & 9.7 \\
SPECTRE (~\cite{Martinkus2022SPECTRE}) & GAN       & 25.0 & 3.0   & --- & ---   & 52.5 & 2.2 \\
DiGress (~\cite{Vignac2022DiGress})    & Diffusion & 77.5 & 5.1   & 90.0 & \textbf{1.6} & 60.0 & \textbf{1.7} \\
EDGE (~\cite{Chen2023EDGE})         & Diffusion & 0.0  & 431.4 & 0.0  & 850.7 & 0.0  & 51.4 \\
BwR (EDP\mbox{-}GNN) (~\cite{Diamant2023Bandwidth}) & Diffusion & 0.0 & 251.9 & 0.0 & 11.4 & 7.5 & 38.6 \\
BiGG (~\cite{Dai2020ScalableSparseGraphs})     & Autoregressive & 5.0  & 16.0  & 75.0 & 5.2   & 10.0 & 11.9 \\
GraphGen (~\cite{Goyal2020GraphGen}) & Autoregressive & 7.5 & 210.3 & 95.0 & 33.2  & 5.0  & 48.8 \\
HSpectre (~\cite{Bergmeister2023ILE}) & Diffusion & 95.0 & 2.1 & \textbf{100.0} & 4.0 & 75.0 & 10.5 \\
GruM (~\cite{Jo2024DiffusionMixture})      & Diffusion & 90.0 & 1.8 & --- & --- & 85.0 & \textbf{1.1} \\
CatFlow (~\cite{Eijkelboom2024VFM}) & Flow & 80.0 & --- & --- & --- & 85.0 & --- \\
DisCo (~\cite{Xu2024DiSco})     & Diffusion & 83.6$\pm$2.1 & --- & --- & --- & 66.2$\pm$1.4 & --- \\
Cometh (~\cite{Siraudin2024Cometh}) & Diffusion & 99.5$\pm$0.9 & --- & --- & --- & 75.0$\pm$3.7 & --- \\
DeFoG (5\% steps) & Flow & 95.0$\pm$3.2 & 3.2$\pm$1.1 & 73.5$\pm$9.0 & 2.5$\pm$1.0 & 86.5$\pm$5.3 & 2.2$\pm$0.3 \\
\rowcolor{gray!12}
DeFoG (~\cite{Qin2024DeFoG})& Flow & 98.5$\pm$1.0 & 1.4$\pm$0.4 & 96.5$\pm$2.6 & 1.4$\pm$0.4 & 90.0$\pm$5.1 & 4.9$\pm$1.3 \\
\midrule
\rowcolor{gray!12}
EQUIMF (Our method)& Flow & \textbf{99.6}$\pm$1.0 & \textbf{1.6}$\pm$0.1 & 97.2$\pm$2.2 & \textbf{1.6}$\pm$0.2 & \textbf{91.2}$\pm$3.0 & 4.9$\pm$2.0\\
\bottomrule
\end{tabular}
% \end{center}
\end{table*}

\paragraph{Setup.}
As we know, our method is a union of discrete and continuous domains; we test the discrete generation performance in this section. Following the ~\cite{Qin2024DeFoG}, we evaluate on standard synthetic graph benchmarks that cover diverse structural patterns, including Planner, SBM~\cite{spectre2022}, Tree datasets~\cite{iterative2024}.
we test the common metrics that valid, unique, and novel. The baseline are all followed ~\cite{Qin2024DeFoG}.

\paragraph{Results and Analysis.}
The results are presented in Table \ref{tab:synthetic}, where performance is measured by VUN and structural Ratio metrics, averaged over five independent runs. The consistent performance gains across all three synthetic datasets confirm the effectiveness of our discrete graph geometry.

\section{Sample Distortion}\label{app:appendixD}
In our project, we adopt a time distortion strategy similar to the one proposed in ~\cite{Qin2024DeFoG}.The core idea is to apply a non-uniform time step discretization during the sampling process, particularly emphasizing critical regions where fine-grained control is needed. This time distortion function address the issue where uniform time step discretization may not preserve essential properties of the graph during critical stages of sampling, such as when fine local variations are crucial for global graph characteristics like planarity or connectivity. In particular, we define a distortion function \( f_\text{dist}(t) \) that increases the granularity of time steps during the most critical stages of the graph evolution, such that:
\begin{align*}
    t' = f_\text{dist}(t) \quad \text{where} \quad f_\text{dist}(t) \in [0, 1],
\end{align*}
is a strictly increasing function that stretches the final parts of the evolution process to capture more subtle structural changes in the graph. For instance, one such function we apply is:
\[
    f_\text{dist}(t) = 2t - t^2,
\]
which accelerates sampling during the initial phase and slows down the final parts to capture critical graph characteristics. This is similar to the polynomial distortion used in \cite{Qin2024DeFoG}.The sample distortion strategy, designed to emphasize key transitions in graph dynamics, improves our model’s sensitivity to intricate local changes. 

\section{Experimental Details}\label{app:appendixE}
For the dataset, we primarily use the publicly available dataset QM9~\cite{ramakrishnan2014quantum} and GEOM-DRUG dataset~\cite{axelrod2022geom}. Our experiments are implemented with a Pytorch based architecture, and the experimental environment is Ubuntu 22.04. For computational resources, we use an PRO6000 96GB GPU. Detailed hyperparameter settings are provided in Table.

\begin{table}[H]
    \centering
    \caption{Hyperparameter Settings}
    \begin{tabular}{lc}
        \toprule
        Hyperparameter & Value \\
        \midrule
        Batch Size & 64 \\
        Optimizer & Adam \\
        Learning Rate & \(1 \times 10^{-4}\) \\
        Hidden Layer & 9 \\
        Hidden Dimension & 256 \\
        Distortion function & Ploydec\\
        iteration & 1000 \\
        $\lambda_{\text{cont}}$ & 0.2 \\
        $\lambda_{\text{disc}}$ & 0.8 \\
        NFE & 50 \\
        \bottomrule
    \end{tabular}
    \label{tab:hyperparams}
\end{table}

\section{Shared Representation Encoder}\label{app:appendixF}
\subsection{Shared Representation Encoder}\label{sec:4.2}
To enable tight coupling between discrete graph structure and continuous 3D geometry, we introduce a shared \(\mathbf{SE}\)(3)-equivariant backbone encoder $\phi_{\theta_1}$ that unifies the feature extraction and cross-modal information fusion for both generation heads. The encoder takes $\mathbf{G}_t$ and a time embedding $\tau = \text{MLP}_t(t)$ that encodes the temporal stage of the noising process as input. The encoder proceeds in three stages.
(1) Input Embedding: Discrete features are projected into a high-dimensional latent space, with time embedding $\tau$ added to ensure temporal consistency.
(2) Equivariant Message Passing: A stack of $L$ EGNN layers performs message passing that operates on relative atomic coordinates and invariant features, thereby preserving \(\mathbf{SE}\)-equivariance.
(3) Feature Branching: The output of the EGNN stack is split into two complementary branches: an \(\mathbf{SE}\)(3)-invariant structural feature $\mathbf{h}_t^{\text{disc}} \in \mathbb{R}^{d_h}$, obtained by pooling node features and refining with the global graph feature, and an \(\mathbf{SE}\)(3)-equivariant geometric feature $\mathbf{h}_t^{\text{cont}} \in \mathbb{R}^{N \times d_h}$, derived directly from the node-level EGNN outputs, where 
\begin{align*}
    \mathbf{h}_t = [\mathbf{h}_t^{\text{disc}}, \mathbf{h}_t^{\text{cont}}] = \phi_{\theta_1} \left(\mathbf{G}_t, t\right) .
\end{align*}
The invariant structural feature $\mathbf{h}_t^{\text{disc}}$, which implicitly encodes geometric information from $\mathbf{R}_t$ via the shared encoder is fed to the discrete MeanFlow head to condition graph evolution on geometry. Meanwhile, the equivariant geometric feature $\mathbf{h}_t^{\text{cont}}$, which implicitly encodes structural information from $(\mathbf{X}_t, \mathbf{E}_t)$ — is passed to the continuous MeanFlow head to inform velocity field prediction with structural context. This design ensures that both generation heads operate on a unified, information-rich latent representation, laying the foundation for mutually conditioned joint generation.

\end{document}